%% file: main2.tex
\documentclass[11pt]{article}
\usepackage{scribe}

\usepackage[usenames,dvipsnames]{xcolor}

\newcommand{\Tribes}{\mathsf{Tribes}}
\newcommand{\into}{\rightarrow}
\newcommand{\bx}{{\boldsymbol{x}}}

\newcommand{\CycleRun}{{\sf CycleRun}}
\newcommand{\CR}{\CycleRun}

\usepackage[backref, colorlinks,citecolor=blue,bookmarks=true]{hyperref}

\usepackage[usenames,dvipsnames]{xcolor}
\makeatletter
\newtheorem*{rep@theorem}{\rep@title}
\newcommand{\newreptheorem}[2]{
\newenvironment{rep#1}[1]{
 \def\rep@title{#2 \ref{##1}}
 \begin{rep@theorem}\itshape}
 {\end{rep@theorem}}}
\makeatother

\newreptheorem{theorem}{Theorem}
\newreptheorem{lemma}{Lemma}
\newreptheorem{proposition}{Proposition}

\newcommand{\AC}{\mathsf{AC^0}}
\newcommand{\OPT}{\mathsf{OPT}}
\newcommand{\C}{{\mathcal C}}

\newcommand{\etal}{et al.\ }

\title{Approximate resilience, monotonicity, and the complexity of agnostic learning \author{Dana Dachman-Soled \\University of Maryland \and Vitaly Feldman \\ IBM Research - Almaden  \and Li-Yang Tan\thanks{Supported by NSF grants CCF-1115703 and CCF-1319788.} \\ Columbia University  \and Andrew Wan \\ Simons Institute, UC Berkeley \and Karl Wimmer\thanks{Supported in part by NSF-CCF-1117079.  Most of this work was done while the author was visiting Simons Institute for the Theory of Computing, University of California-Berkeley} \\ Duquesne University}}
\begin{document}
\date{}
\maketitle
\thispagestyle{empty}
\begin{abstract}
A function $f$ is $d$-resilient if all its Fourier coefficients of
degree at most $d$ are zero, i.e.~$f$ is uncorrelated with all
low-degree parities. We study the notion of \emph{approximate
resilience} of Boolean functions, where we say that $f$ is
$\alpha$-approximately $d$-resilient if $f$ is $\alpha$-close to a
$[-1,1]$-valued $d$-resilient function in $\ell_1$ distance. We show that
approximate resilience essentially characterizes the complexity of
agnostic learning of a concept class $\C$ over %product distributions.
the uniform distribution.
Roughly speaking, if all functions in a class $\C$
are far from being $d$-resilient then $\C$ can be learned
agnostically in time $n^{O(d)}$  and conversely, if $\C$ contains
a function close to being $d$-resilient then agnostic learning of $\C$
in the statistical query (SQ) framework of Kearns has complexity of at
least $n^{\Omega(d)}$. This characterization is based on the duality between $\ell_1$
approximation by degree-$d$ polynomials and approximate $d$-resilience that we establish.
In particular, it implies that $\ell_1$ approximation by low-degree polynomials, known to be sufficient for agnostic learning over product distributions, is in fact necessary.

Focusing on monotone Boolean functions, we exhibit the existence of near-optimal $\alpha$-approximately $\widetilde{\Omega}(\alpha\sqrt{n})$-resilient monotone functions for all $\alpha>0$. Prior to our work, it was conceivable even that every monotone function is $\Omega(1)$-far from any $1$-resilient function. Furthermore, we construct simple, explicit monotone functions based on $\Tribes$ and ${\sf CycleRun}$ that are close to highly resilient functions. Our constructions
are based on general resilience analysis and amplification techniques we introduce. These structural results, together with the characterization, imply nearly optimal lower bounds for agnostic learning of monotone juntas, a natural
variant of the well-studied junta learning problem. In particular we show that no SQ algorithm can efficiently agnostically learn monotone $k$-juntas for any $k = \omega(1)$ and any constant error less than $1/2$.
 \end{abstract}

%\newpage
\setcounter{page}{1}

\input{intro2.tex}
\input{prelims.tex}
\input{sqbound.tex}

\input{tribes.tex}

\section{Conclusions}
%The In this work we have de
We have demonstrated that complexity of agnostic learning over product distributions has a natural characterization via either of two dual notions: $\ell_1$-approximation by polynomials and approximate resilience. The notion of distance to resilience that we introduce appears to be interesting its own right. It is also better suited for proving lower bounds since a single close resilient function witnesses the hardness of agnostic learning.
Our proof of this result is relatively simple and remarkably, up to the choice of norms, is identical to Sherstov's powerful pattern matrix method in communication complexity \cite{Sherstov:11pmm}.

An application of our characterization and our second contribution is new and detailed picture of the hardness of agnostic learning of monotone functions over the uniform distribution. Some evidence that agnostic learning of several monotone classes is hard is already known and relies on cryptographic assumptions  \cite{KKMS:08,FeldmanGKP09,KlivansS09}. Yet the existing evidence is restricted to the very hard regime when $\OPT$ is near $1/2$ and does exclude learning with excess error of just $1\%$ that would suffice for most practical applications. We give the first general lower bounds for monotone functions that establish hardness in the low-error regime. We also describe simple and explicit monotone functions that are very close to being resilient.

Finally, we give general tools for analysis of approximate resilience. Such tools might find use for proving new agnostic learning lower bounds.

\section*{Acknowledgements}
Theorem \ref{thm:duality} and a special case of Theorem \ref{thm:lb-ar} for symmetric functions were first derived in V.F.'s collaboration with Pravesh Kothari. We thank Pravesh for his permission to include the result in this work.
%Theorem \ref{thm:duality} and special cases of Theorem \ref{thm:lb-ar} were derived in V.F.'s collaboration with Pravesh Kothari and A.W.'s collaboration with Justin Thaler (independently). We thank Pravesh and Justin for their permission to include the result in this work.
%Partial
We thank Justin Thaler for his help in deriving Theorem \ref{thm:duality} and illuminating discussions on the relationship of our characterization of agnostic learning to the pattern matrix method of Sherstov \cite{Sherstov:11pmm}. We thank Udi Wieder and Yuval Peres for helpful information about the $\CycleRun$ function, and Ryan O'Donnell, Johan H{\aa}stad, Rocco Servedio and Jan Vondr\'{a}k for helpful conversations.

\bibliographystyle{alpha}
\bibliography{agnosticmono}
\appendix
\input{product.tex}
\input{append.tex}
\input{cyclerun.tex}
\input{cycleruninf.tex}
\end{document}

%% file: intro2.tex
% !TEX root =main2.tex
\newcommand{\B}{{\mathcal B}}
\newcommand{\cP}{{\mathcal P}}

\section{Introduction}
The agnostic learning framework \cite{Haussler:92,KSS:94}, models learning from examples in the presence of worst-case noise.
In this framework the learning algorithm is given random examples $(\x,f(\x))$ where $\x$ is chosen from some distribution $D$ and $f$ is an \textit{arbitrary} Boolean function. The goal of the agnostic learning algorithm for a concept class $\C$ is to output a hypothesis $h$ that agrees with $f$ almost  as well as the best function in $\C$; that is:
$$\Pr_D[h(\x) \neq f(\x) ] \leq \min_{c\in \mathcal{C}} \Pr_D[c(\x)\neq f(\x)] + \eps,$$
where $\eps$ is an error parameter given to the algorithm.

Understanding the complexity of learning in the agnostic model is central to both theory and practice in machine learning research. Learning in this model is notoriously hard, and despite two decades of intensive research our formal understanding of the complexity of agnostic learning is still very limited. Even when $D$ is the uniform distribution over $\bn$, agnostic learning has proven extremely challenging: few non-trivial classes are known to be learnable agnostically. The primary technique used for agnostic learning in this setting is the polynomial $\ell_1$ regression algorithm introduced in the influential work of Kalai et al \cite{KKMS:08}.
% only meaningful algorithm for this setting is the
%There are no known efficient algorithms even for learning simple classes such as Boolean disjunctions and even when $D$ is the uniform distribution on the hypercube. Aside from the classes of constant VC-dimension (which can be learned easily even with respect to arbitrary distributions), the only known technique for agnostic learning under the uniform distribution is the
This algorithm finds a low-degree polynomial that minimizes the $\ell_1$ distance to the target function, and can be applied to agnostically learn classes which are well approximated by polynomials. This approach has lead to the first agnostic learning algorithm for $\AC$ circuits (in quasi-polynomial time) and halfspaces (in $n^{O(1/\eps^2)}$ time) over the uniform distribution \cite{KKMS:08} and was used in many other agnostic learning results.
%
%There are no known efficient algorithms even for learning simple classes such as Boolean disjunctions and even when $D$ is the uniform distribution on the hypercube. Aside from the classes of constant VC-dimension (which can be learned easily even with respect to arbitrary distributions), the only known technique for agnostic learning under the uniform distribution is the polynomial $\ell_1$ regression algorithm introduced in the influential work of Kalai et al \cite{KKMS:08}.
%This algorithm finds a low-degree polynomial that minimizes the $\ell_1$ distance to the target function, and can be applied to agnostically learn classes which are well approximated by polynomials. This approach has lead to the first agnostic learning algorithm for $\AC$ circuits (in quasi-polynomial time) and halfspaces (in $n^{O(1/\eps^2)}$ time) over the uniform distribution \cite{KKMS:08} and was used in many other agnostic learning results.

In this work we address the complexity of agnostic learning relative to the uniform and, more generally, product distributions. In addition to running time, a critical but often unstated parameter in lower bounds on agnostic learning is the value of $\OPT_\C(D,f)=\min_{c\in \mathcal{C}} \Pr[c(\x)\neq f(\x)]$ to which the lower bound applies (note that $\OPT$ is essentially the noise rate). If a hardness result requires learning functions $f$ for which $\OPT_\C(D,f)$ is close to $1/2$, then it does not apply to most practical learning applications. (If $\C$ does not have any useful classifiers, it does not make much sense to use $\C$ as a performance benchmark.)
%In practice it does not make much sense to compare the performance to that of classifiers in $\C$ when $\C$ does not have any useful classifiers.
Therefore it is more important to understand the complexity of agnostic learning in which $\OPT$ is a small constant close to 0 (or even approaches 0 as $n$ grows). However essentially all known lower bounds for agnostic learning are in the hardest regime when $\OPT_\C(D,f)$ goes to $1/2$ as dimension and other problem parameters grow
(although there are some notable exceptions in restricted models and the more challenging distribution-independent setting \cite{KlivansSherstov:10,FeldmanGRW:12}). In this work we aim to precisely characterize the value of $\OPT$ for which agnostic learning becomes hard and therefore will make this parameter explicit in our lower bounds.

In machine learning literature it is more common to specify the {\em excess error} which is the difference between $\OPT_\C(D,f)$ and the error of the produced hypothesis that an algorithm can achieve. It is easy to see that lower bounds showing that excess error of $\kappa$ cannot be achieved is equivalent to stating that the lower bound applies to a setting where $\OPT = 1/2-\kappa$ (since error of $1/2$ can always be achieved).

\subsection{Approximate resilience and agnostic learning}
In this work we explain why the polynomial $\ell_1$ regression algorithm is the best approach known to date for agnostically learning over product distributions. Specifically, we prove that the complexity of agnostic learning $\calC$ over a product distribution in the statistical query model is characterized by how well $\calC$ can be approximated in the $\ell_1$ norm by low-degree polynomials over the same distribution. The statistical query (SQ) model \cite{Kearns:98} is a well-studied restriction of the PAC learning model in which the learner relies on approximate expectations of functions of an example rather than examples themselves. With the exception of Gaussian elimination\footnote{Note that Gaussian elimination fails in the presence of even minor amounts of random noise and is not applicable in the agnostic framework.} all known techniques used in the theory and practice of machine learning have statistical query analogues. Polynomial $\ell_1$ regression is no exception,
%which can implemented using statistical queries via a reduction to learning halfspaces \cite{BlumFKV:97,DunaganVempala:04} or using agnostic boosting \cite{KalaiKanade:09,Feldman:10ab}.
and therefore to prove our characterization it suffices to establish a lower bound on learning by statistical query algorithms for function classes that are not well-approximated by low-degree polynomials.

The optimality of $\ell_1$ regression for agnostic learning over product distributions that we prove is based on a formal connection between agnostic learning and a basic structural property of Boolean functions. We say that a function $g : \bn\to \R$ is $d$-resilient if $\hat{g}(S) = 0$ for all $|S| \le d$, {\it i.e.}~$g$ is uncorrelated with every low-degree parity. Equivalently, $g$ is $d$-resilient if and only if $\E[g_\rho]=\E[g]$
%$\E[g] = 0$ and
for any restriction $\rho$ to at most $d$ out of $n$ variables and $\E[g]=0$.
Functions which satisfy the first property are called {\em correlation immune}
%This property is also referred to as {\em correlation immunity}
and are widely-studied for cryptographic applications. The structural question we will be interested in is:
\begin{center}
{\it How close can a Boolean function be to a highly resilient function with range in $[-1,1]$?}
\end{center}
More precisely, we say that $f:\bn\to [-1,1]$ is $\alpha$-approximately $d$-resilient if there exists a $d$-resilient $g:\bn\to [-1,1]$ such that $\| f -g \|_1 = \E[|f(\x)-g(\x)|]\le \alpha$, and we will be interested in functions that are $\alpha$-approximately $d$-resilient for small values of $\alpha$ and large values of $d$. We note that for simplicity and convenience the definitions here are for the uniform distribution on the hypercube but can be easily extended to general product distributions over other $n$-dimensional domains (see Section \ref{sec:product}).

The notion of resilience is well-studied and has applications in cryptography, pseudorandomness, inapproximability, circuit complexity and more (for a few examples, see \cite{chor1985bit,luby1995pairwise,austrin2009approximation,austrin2011randomly,Sherstov:11pmm}). However, to the best of our knowledge our notion of approximate resilience does not appear to have been explicitly studied before.

At a high level we show that if a concept class $\C$ contains an $\alpha$-approximately $d$-resilient function then the complexity of learning $\C$ agnostically in the SQ model is $n^{\Omega(d)}$. Further, learning is hard even for $\OPT \leq \alpha/2$ (in other words when noise rate is $\alpha/2$). For simplicity the complexity of an SQ algorithm refers to a polynomial upper-bounding both the running time and the inverse of query tolerance. Naturally, the presence of a single $\alpha$-approximately $d$-resilient function would not suffice for a hardness result since a concept class with a single function can be easily learned agnostically. We therefore need some assumptions under which existence of a single $\alpha$-approximately $d$-resilient function will imply that there are many of them. One such assumption that we adopt is that the $\alpha$-approximately $d$-resilient function $c$ depends on at most $n^{1/3}$ variables (such a function is called a $n^{1/3}$-junta) and the concept class $\C$ is closed under renaming of variables. Alternatively, if we consider an ensemble of concept classes $\{\mathcal{C}_n\}_{n=1}^{\infty}$ parameterized by dimension $n$ it would be sufficient to assume that the ensemble is closed under addition of irrelevant variables. For brevity we omit the closed-ness under renaming since it is satisfied by all commonly-studied concept classes. We now state our lower bound in terms of resilience informally.

\begin{theorem}\label{thm:lb-ar}
Let $\C$ be a concept class. Fix $d$ and let $\alpha(d)$ be such that, there exists a $\alpha(d)$-approximately $d$-resilient $n^{1/3}$-junta $c\in \calC$. Then any SQ algorithm for agnostically learning $\C$ with excess error of at most $\frac{1- \alpha(d)}{2} - n^{-o(d)}$ has complexity of at least $n^{\Omega(d)}$.
\end{theorem}
Alternatively, this result can be stated as saying that if for every function $f$ satisfying $\OPT_\C(D,f) \leq \alpha(d)/2$ the algorithm outputs $h$ such that $\Pr_D[h(\x) \neq f(\x) ] \leq 1/2 - n^{-o(d)}$ then its SQ complexity is $n^{\Omega(d)}$. An immediate implication of this theorem is that a concept class containing an $o(1)$-approximately $d$-resilient function cannot be learned with noise rate larger than $o(1)$ in time $n^{\Omega{(d)}}$.

The proof of this theorem is based on the simple observation that agnostic learning of $\C$ is at least as hard as weak learning of a class of $d$-resilient functions which are close to functions in $\C$. From there we rely on hardness of SQ learning of pairwise nearly orthogonal functions to obtain the claim. This result relies crucially on the distribution being a product distribution and it is was recently demonstrated that is does not hold for some non-product distributions \cite{FeldmanKothari:14}.

The lower bounds obtained from this technique are closest in spirit to lower bounds based on cryptographic assumptions and those based on hardness of learning sparse parities with noise. Cryptographic hardness relies on a certain problem being hard for all known ``attacks". As pointed out above, SQ algorithms capture all known agnostic learning algorithms and learning techniques in general. Therefore the lower bounds hold against all known learning algorithms. Further, as in our lower bounds, degree of resilience of a predicate is the primary hardness parameter in many cryptographic constructions ({\em cf.} \cite{o2013goldreich}).

This simple technique might appear to be a relatively limited approach to obtaining lower bounds. Yet, it turns out that the lower bounds it achieves are essentially optimal. This follows from  the duality between approximate resilience and $\ell_1$ approximation by low-degree polynomials that we establish. More formally, let $\mathcal{P}_d$ be the class of degree at most $d$ real-valued polynomials. For a Boolean function $f$, let $\Delta_{\cP_d}(f) = \min_{p\in \mathcal{P}_d} \E[ |f-p|]$.
\begin{theorem}\label{thm:duality}
For $f:\bits^n \to \bits$ and $0\leq d\leq n$ and $\alpha \geq 0$, $f$ is $\alpha$-approximately $d$-resilient if and only if $\Delta_{\cP_d}(f) \geq 1-\alpha$.
\end{theorem}
The proof of this result is a fairly simple application of a classical result on duality of norms by Ioffe and Tikhomirov \cite{IoffeTihkomirov:68}.

Now for a concept class $\C$, let $\Delta_{\cP_d}(\C) = \max_{f\in \mathcal{C}} \Delta_{\cP_d}(f)$.  To see how this quantity characterizes agnostic learning in the statistical query model, we state the error and running time achieved by the polynomial $\ell_1$ regression algorithm of Kalai \etal for agnostic learning \cite{KKMS:08}. This algorithm is easy to implement in the SQ model\footnote{To the best of our knowledge this is not proved anywhere explicitly but is fairly well-known and used in some other works \cite{CheraghchiKKL:12}. It follows from the fact that LPs can optimized approximately using approximate evaluations of the optimized function (in our case expected $\ell_1$ error) for example via the Ellipsoid algorithm \cite{lovasz1987algorithmic}. See \cite{FeldmanPV:13} for more details on this general technique.}.
%We state the result here explicitly for comparison with the lower bound.
%As we mentioned earlier, the polynomial $\ell_1$ regression algorithm of Kalai \etal implies agnostic learning \cite{KKMS:08}. This algorithm is easy to implement in the SQ model\footnote{To the best of our knowledge this is not proved anywhere explicitly but follows easily from the fact that LPs can optimized approximately using approximate evaluations of the optimized function (in our case expected $\ell_1$ error). \cite{Lovasz book}}. We state the result here explicitly for comparison with the lower bound.
\begin{theorem}[\cite{KKMS:08}] \label{thm:kkms+boost}
Let $\C$ be a concept class over $\bits^n$ and fix $d$. There exists a SQ algorithm which for any $\eps > 0$ agnostically learns $\C$ with excess error $\Delta_{\cP_d}(\C)/2 + \eps$ and has complexity $\poly(n^{d},1/\eps)$.
\end{theorem}
On the other hand, we may apply Theorems \ref{thm:duality} and \ref{thm:lb-ar} to show that this is the best any SQ algorithm can do;
%any SQ algorithm for agnostically learning $\C$ requires essentially the same complexity to achieve the same error.
 by Theorem \ref{thm:duality} there exists an $\alpha(d)$-approximately $d$-resilient function in $\C$ with $1-\alpha(d) = \Delta_{\cP_d}(\C).$ Therefore Theorem \ref{thm:lb-ar} essentially matches the upper bound of Theorem \ref{thm:kkms+boost} in excess error and complexity, implying the optimality of $\ell_1$-regression based algorithms for agnostic learning over the uniform distribution.
%This implies the optimality of $\ell_1$-regression based algorithms for agnostic learning over the uniform distribution.
The extension to other product distributions is fairly straightforward and we discuss it in Sec.~\ref{sec:product}.

\subsection{Learning monotone juntas}
With this characterization in hand, we would like to better understand what classes of functions we can hope to agnostically learn on the uniform distribution.  Uniform distribution learning is challenging even in the noiseless setting, with efficient algorithms out of reach for natural classes such as polynomial size DNF formulas and decision trees. However, learning monotone functions and their corresponding subclasses seems significantly easier; for example,  monotone decision trees \cite{OdonnellServedio:07} and monotone DNFs with few terms \cite{Servedio:01mdnf} are efficiently learnable in the SQ model (for other examples see \cite{OWimmer:09,BlumBL:98,BshoutyTamon:96}).

This difference is demonstrated most dramatically in the junta learning problem,
%The junta learning problem is
which is considered by many to be the single most important open problem in uniform distribution learning.  In this problem, the target function is an unknown \emph{$k$-junta}, a Boolean function which depends on at most $k\ll n$ variables.
%Aside from being natural in its own right, with close connections to fundamental questions concerning the structure of Boolean functions,
%\footnote{The difficulty here lies in identifying the relevant variables without performing exhaustive search.}
  The junta problem also lies at the heart of the notorious DNF and decision tree learning problems:  Since
$s$-term DNFs and $s$-leaf decision trees can compute arbitrary $(\log s)$-juntas, learning either of these classes requires that we first be able to efficiently learn $\omega(1)$-juntas.
Progress has remained slow in the 20 years since Blum posed the junta problem,  with the current fastest algorithm running in time $n^{.60k}$ \cite{valiant2012finding}, improving on the first non-trivial algorithm which runs in time $n^{.704k}$ \cite{MOS:04short} (the trivial algorithm exhaustively checks all $k$-subsets of $[n]$ and runs in time $O(n^k)$). %\medskip
In contrast, monotone juntas are easy to learn using an extremely simple algorithm:
the relevant variables can be identified by estimating their correlations with the target function $\E[f(\x)\x_i] = \hat{f}(\{i\})$, and thus monotone $k$-juntas can be learned in time $O(n+2^k)$.
Does the advantage of monotonicity hold in the agnostic setting as well?
We first consider the simplest problem of agnostic learning monotone juntas.  While it appears to be a hard problem, known hardness results for specific monotone functions do not rule out polynomial time algorithms for any constant $\eps$. Specifically, the best known lower bound is $n^{\Omega(1/\eps^2)}$ for majority functions \cite{KKMS:08} and is based on the assumption that learning sparse noisy parities is hard. Further, this hardness result only applies when $\OPT \geq 1/2-\eps$ which leaves open the possibility that the problem is solvable efficiently when the noise rate is a constant smaller than $1/2$.

As we saw in Theorem \ref{thm:lb-ar}, the complexity of agnostic learning of $\C$ is characterized by the approximate resilience of functions in $\C$. Therefore we consider the structural question of how close monotone functions are to bounded resilient functions.
The structure of monotone functions over the Boolean hypercube has been investigated in many influential works (see \cite{BlumBL:98,BshoutyTamon:96,mossel2002noise,Odonnell:03thesis,OWimmer:09}). While to the best of our knowledge our notion has not been studied before, several works have examined the total spectral weight that monotone functions have on low-degree coefficients \cite{BshoutyTamon:96,mossel2002noise}. Spectral weight indicates the distance to the closest (not necessarily bounded) %\textit{unbounded}
resilient function in $\ell_2$ norm. Both differences of bounded/unbounded and $\ell_1/\ell_2$ are significant, but we show how bounds on low-degree spectral weight can serve as a basis for bounds on our notion of distance to resilience (see Thm.~\ref{low-degree-resilient}).

It is easy to see that monotone functions cannot be $1$-resilient, and prior to our work, it was possible that every monotone function was $\Omega(1)$-far from $1$-resilient. Our first structural result rules out this possibility in a very strong way:
\begin{theorem}
\label{talagrand-resilient}
For every $\alpha > 0$ there exists an $\alpha$-approximately $d$-resilient monotone Boolean function where $d = \Omega(\alpha \sqrt{n}/\log n)$.
\end{theorem}
Our proof of this result is indirect and relies crucially on the duality of approximate resilience and $\ell_1$-approximation of monotone functions by polynomials. We use a lower bound for PAC learning of monotone functions by Blum \etal \cite{BlumBL:98} to obtain strong lower bounds on $\ell_1$-approximation of monotone functions by polynomials. We can then use Theorem \ref{thm:duality} to obtain bounds on distance to resilience.

This degree of resilience is essentially optimal: combining basic facts from discrete Fourier analysis, it is straightforward to see that every monotone Boolean function is $\alpha$-far from any $\Omega(\alpha\sqrt{n})$-resilient function \cite{BshoutyTamon:96}.
Applying our connection between approximate resilience and agnostic learning, we get as a corollary our main application:
\begin{corollary}
\label{sq-lower-bound}
Any SQ algorithm for agnostically learning the class of monotone $k$-juntas
with excess error of $1/2-\alpha$ has complexity of $n^{\Omega(\alpha\sqrt{k}/\log k)}$.
\end{corollary}
Qualitatively, Corollary \ref{sq-lower-bound} gives the first super-polynomial lower bound on the complexity of SQ algorithms for agnostically learning monotone $k$-juntas with constant (and even sub-constant) noise. It also rules out the possibility of efficient SQ algorithms for agnostic learning monotone decision trees and monotone DNFs with few terms (which, as previously mentioned, do have efficient SQ algorithms in the noiseless setting).
Quantitatively, our lower bound essentially matches the upper bound of $n^{O(\sqrt{k}/\eps)}$ that follows as a corollary of the low-degree concentration bound of \cite{BshoutyTamon:96} and the polynomial $\ell_1$ regression algorithm \cite{KKMS:08}.
%\vnote{We need to also mention a reference for Fourier concentration bound.}.
Note that lower bounds on PAC learning of monotone functions \cite{BlumBL:98} cannot be translated directly to lower bounds in the junta learning setting since these lower bounds are subexponential in $k$ while junta learning algorithms are allowed to run in time polynomial in $2^k$.

While Theorem \ref{talagrand-resilient} yields a near-optimal lower bound on the complexity of agnostically learning general monotone juntas, the construction is not explicit: it is based on a randomized DNF construction (similar to Talagrand's randomized DNF construction \cite{Talagrand:96}), and contains functions of high complexity.  Furthermore, for more general classes such as monotone DNFs, the hardness results implied are not optimal.
We first show that even the simple $\Tribes$ function, a read-once DNF, is close to a resilient function (which gives a stronger hardness result for learning small monotone DNFs).

\begin{theorem}
\label{tribes-resilient}
$\Tribes$ is $\alpha$-approximately $d$-resilient, where $\alpha =  O(n^{-1/3})$ and $d=\Omega(\log n/\log \log n)$.
\end{theorem}
%\noindent Next we show that the resilience in Theorem \ref{tribes-resilient} can be amplified to $2^{\Omega(\sqrt{\log n})}$ in an explicit way by iteratively composing $\Tribes$ with itself.

Our proof of Theorem \ref{tribes-resilient} is based on a general technique for obtaining bounds on approximate resilience from bounds on spectral weight on low-degree coefficients. Roughly, our result states that for a sufficiently small $\gamma$, if the total spectral weight on degree $\leq d$ coefficients of $f$ is at most $\gamma$, then $f$ is $\approx \sqrt{\gamma} e^d$-approximately $d$-resilient (see Thm.~\ref{low-degree-resilient}). The proof relies on a concentration inequality for low-degree polynomials over independent Rademacher random variables that follows from the hypercontractivity inequalities of Bonami and Beckner \cite{Bon70,Bec75}.

We then describe a general technique for amplifying the degree of approximate resilience of functions via iterative composition and apply it to $\Tribes$ to obtain an explicit function that is $o(1)$-approximately $2^{\Omega(\sqrt{\log n})}$-resilient (see Section \ref{sec:amplify} for details).

%The resilience in Theorem \ref{tribes-resilient}  in an explicit way by iteratively composing $\Tribes$ with itself .

%quite general: we show that the projection of a Boolean function onto its high-degree part can be transformed into a bounded function which is both resilient and close to the original function, as long as it has $o(1)$

Both Theorems \ref{talagrand-resilient} and \ref{tribes-resilient} give monotone Boolean functions which are close to resilient functions, however the resilient functions are not necessarily Boolean-valued. In most cryptographic applications resilience is studied specifically for Boolean functions ({\em e.g.,} \cite{Siegenthaler:84,MOS:04short,o2013goldreich}), and therefore it is natural to ask if
there are such functions that are close to monotone Boolean functions.
%there are monotone Boolean functions which are close to resilient Boolean-valued functions.
Using a new function called $\CycleRun$~\cite{cyclerun}, we show that this is indeed possible, and furthermore we nearly match the resilience of the iterated $\Tribes$ construction:
\begin{theorem}
\label{cycle-run-resilient}
There is an explicit $\alpha$-approximately $d$-resilient  monotone Boolean function $f$ where $\alpha = o_n(1)$ and $d = 2^{\Omega(\sqrt{\log n}/\log\log n)}$. Furthermore, $f$ is $\alpha$-close to a Boolean $d$-resilient function.
\end{theorem}

We prove Theorem \ref{cycle-run-resilient} by first showing that $\CycleRun$ is
$O(\sqrt{\log n/n})$-approximately $1$-resilient, where our witness to this approximate resilience is a Boolean function.  Our argument crucially relies on four key properties of $\CycleRun$: monotonicity, low influence, oddness, and invariance under cyclic shifts; as far as we know, $\CycleRun$ is the only explicit Boolean function known to have all four properties.  These properties allow us to use a structured combinatorial argument, unlike our argument for $\Tribes$ that relies on properties of polynomials and produces a witness that is a bounded function (and applying this style of argument to $\Tribes$ quickly gets unruly).  Having established $O(\sqrt{\log n/n})$-approximate $1$-resilience, we then apply the aforementioned general amplification technique to increase the degree of resilience to $2^{\widetilde{\Omega}(\sqrt{\log n})}$.

We remark that while the degrees of resilience obtained in Theorems \ref{cycle-run-resilient} and \ref{tribes-resilient} are not as strong as that of Theorem \ref{talagrand-resilient}, both are sufficient to rule out the existence of efficient SQ algorithms for learning monotone $k$-juntas for any $k=\omega_n(1)$ and subconstant error-rate.

\subsection{Related work}
Lower bounds for statistical query algorithms were first shown by Kearns \cite{Kearns:98} who proved that parities cannot be learned by SQ algorithms. Soon
after this Blum et al.~\cite{blum1994weakly} characterized the weak PAC
learnability of every function class ${\cal C}$ in the SQ model in
terms of the \emph{statistical query dimension} of ${\cal C}$;
roughly speaking, this is the largest number of functions from
${\cal C}$ that are pairwise nearly orthogonal to each other (we
give a precise definition in Section~\ref{sec:learn}). These lower bound techniques were extended to strong PAC learning and agnostic learning in more recent work \cite{Simon:07,Feldman:12jcss,szorenyi2009characterizing}. Lower bounds for SQ algorithms were proved for many learning problems including, for example, PAC learning of juntas \cite{blum1994weakly}, weak-learning of intersections of halfspaces \cite{klivans2007unconditional} and learning of monotone depth-3 formulas \cite{FeldmanLS:11colt}. These lower bounds are information-theoretic but capture remarkably well the computational hardness of learning problems. In some cases, such as learning juntas over the uniform distribution, this is the only known formal evidence of the hardness of the problem.

Given the lack of general lower bounds for several basic problems in agnostic learning, many works concentrate on lower bounds against specific popular algorithms such as $\ell_1$-regression \cite{KlivansSherstov:10} and margin-based linear methods  \cite{LongS11,Ben-DavidLSS12,DanielyLS14}. These techniques are captured by SQ algorithms and therefore our lower bounds are substantially more general.

Several previously known lower bounds for agnostic learning are based on the reduction to learning of $k$-sparse noisy parities. This is a notoriously hard problem for which the only non-trivial algorithm is the recent breakthrough result of Valiant that gives
an algorithm running in time $n^{0.8k}$ \cite{valiant2012finding}. Assuming that this problem requires $n^{\Omega(k)}$ time we get that agnostic learning of majorities on the uniform distribution requires $n^{\Omega(1/\eps^2)}$ time \cite{KKMS:08} and conjunctions require $n^{\Omega(\log(1/\eps))}$ time \cite{Feldman:12jcss}. Learning $k$-sparse parities in the SQ model has complexity of $n^{\Omega(k)}$ and therefore these results also give unconditional SQ lower bounds. These lower bounds can be interpreted as special cases of our approach. They are based on showing that a parity of high-degree has a significant correlation with a function in $\C$. Clearly a $k$-sparse parity function is $(k-1)$-resilient and correlation implies that distance to that parity is slightly better than the trivial 1. The main limitation of this approach is that in most cases it can only lead to hardness results when the noise rate is close to $1/2$. In particular this approach cannot lead to the strong hardness results we prove here for monotone juntas.

In a recent work Feldman and Kothari \cite{FeldmanKothari:14} show that the equivalence between $\ell_1$ approximation by polynomials and agnostic learning does not extend to non-product distributions. They exhibit a distribution $D$ for which any polynomial that is $1/3$-close to the disjunction of all the variables in $\ell_1$ (measured relative to $D$) must have degree $\Omega(\sqrt{n})$. At the same time disjunctions are SQ learnable in time $n^{O(\log(1/\eps))}$ over that distribution.

Our approach to proving lower bounds is closest in spirit and shares technical elements with the influential pattern matrix method of Sherstov \cite{Sherstov:11pmm}. His method shows that lower bounds on the approximation by polynomials in $\ell_\infty$ norm of a function $f$ can be translated into lower bounds on randomized communication complexity of a certain communication problem corresponding to evaluation of $f$ on different subsets of variables (which were previously thought as stronger than lower bounds on approximation in $\ell_\infty$ by polynomials). A crucial step in his result is an application of duality that is in some sense symmetric to ours and shows the existence of an unbounded resilient function $g$ that is correlated with $f$. Such $g$ then serves to upper bound discrepancy for the communication problem (from which a lower  bound on randomized communication complexity follows).

%% file: prelims.tex
\subsection{Preliminaries}\label{sec:prelims}
% VIT need to use (D,f) with opt in the intro 

All probabilities and expectations are with respect to the uniform distribution unless otherwise stated, and we will use boldface (e.g.~$\x$ and $\y$) to denote random variables. %\begin{definition}
%We say that a Boolean function $f:\bn\to\bits$ is a $k$-junta if there exists a Boolean function $g:\bits^k\to\bits$ and coordinates $i_1,\ldots,i_k\in [n]$ such that $f(x_1,\ldots,x_n) = g(x_{i_1},\ldots,x_{i_k})$ for all $x\in\bn$.
%\end{definition}
%
%\begin{definition}
Given $f,g:\bits^n\to\R$, we say that $f$ and $g$ are $\eps$-close if $\| f - g \|_1 = \E[|f(\x)-g(\x)|] \le \eps$.
We say that $g$ is bounded if it takes values in the interval $[-1,1]$.
Note that if $f$ is Boolean valued and $g$ is bounded, then $\|f - g\|_1 = 1-\E[fg]$.
Every function $g : \bits^n\to\R$ can be uniquely written as a multilinear polynomial
such that $g(x) = \displaystyle\sum_{S \subseteq [n]} \widehat{g}(S) \displaystyle\prod_{i \in S} x_i$ for all $x \in \{-1,1\}^n$; the coefficients $\hat{g}(S)$ are called the Fourier coefficients of $g$. The total influence of a Boolean function $f:\{-1,1\}^n \to \{-1,1\}$, denoted $\Inf[f]$, is $\sum_{i=1}^n \Pr[f(\bx) \ne f(\bx^{\oplus i})]$, where $x^{\oplus i}$ denotes $x$ with its $i$-th coordinate flipped. 
%\end{definition}
\begin{definition}
A function $g:\bits^n\to \R$ is $d$-resilient if $\widehat{g}(S) = 0$ for all $|S| \le d$.  We say that a Boolean function $f:\bits^n\to\bits$ is $\alpha$-approximately $d$-resilient if there exists a $d$-resilient bounded function $g$ such that $\|f - g\|_1 \leq \alpha$.
\end{definition}
\noindent \textbf{Learning background }  In the agnostic learning framework, the learning algorithm is given labeled examples $(\bx,\boldsymbol{y})$ where $\bx\in \bits^n$ and $\boldsymbol{y}\in \bits$ are drawn from a distribution $\mathcal{D}$ over $\bits^n\times \bits$. As usual we describe such distributions by a pair $(D,g)$, where $D$ is the marginal distribution on $\bits^n$ and $g:\bits^n \to [-1,1]$, where $g(x)=\E_{(\bx,\boldsymbol{y})\sim \mathcal{D}}[ \, \boldsymbol{y}\mid \bx = x\,]$ is expectation of the label for each input. Note that for every Boolean function $f$, if $U$ denotes the uniform distribution then $\E_{(\bx,\boldsymbol{y})\sim (U,g)}[f(\bx) \neq \boldsymbol{y}] = \|f - g\|_1/2$.

\begin{definition}
Let $\C$ be a class of Boolean functions on $\bits^n$. An algorithm $A$ agnostically learns $\C$ over distribution $D$ on $\bits^n$ if for any $g:\bits^n\to [-1,1]$ and $\eps > 0$, given examples from distribution ${\cal D}=(D,g)$ and $\eps$, it outputs with probability at least $2/3$ hypothesis $h:\bits^n\to \bits$ such that:
$$ \Pr[ h(\bx)\neq \boldsymbol{y}] \leq \OPT_\C(D,g) + \eps,$$
where $\OPT = \min_{c\in \C} \Pr_{(\bx,\boldsymbol{y}) \sim (D,g)}[c(\bx)\neq \boldsymbol{y}]$. The algorithm is said to learn with {\em excess error} $\kappa$ if $h$ instead satisfies
$$\Pr[ h(\bx)\neq \boldsymbol{y}] \leq \OPT_\C(D,g) + \kappa.$$
\end{definition}
%We say that $A$ weakly agnostically learns $\mathcal{C}$ if there exists a constant $c$ such that
%$$ \E[h(x) = y] = \E[ h(x) g(x) ]  \geq n^{-c}\cdot \operatorname{OPT_{adv}},$$
%where $\operatorname{OPT_{adv}} = \max_{c\in \mathcal{C}} \E[ c(x) g(x)].$

\begin{definition}
A statistical query is defined by a bounded function of an example $\phi: \bits^n \times \bits \to [-1,1]$ and positive tolerance $\tau$. A valid reply to such a query relative to a distribution ${\cal D}$ over examples is a value $v$ that satisfies:
$$|\E_{(\bx,\boldsymbol{y})\sim {\cal D}}[\phi(\bx,\boldsymbol{y})] - v| \leq \tau.$$
\end{definition}
%A \textbf{statitical query algorithm} is one which receives no labeled examples and instead makes statistical queries to the target function.
A statistical query learning algorithm is an algorithm which relies solely on statistical queries and does not have access to actual examples. We say that an SQ algorithm has \textbf{statistical query complexity} $T$ if
it makes at most $q$ statistical queries of tolerance at least $\tau$ and  $T \geq \max\{q,1/\tau\}$. % The error of the algorithm is the largest error achieved by its output hypothesis when it receives valid answers to all of its statistical queries.

%% file: sqbound.tex
\section{Characterization of Agnostic Learning}\label{sec:learn}
In this section we show that approximate resilience implies hardness of agnostic learning for statistical query algorithms (Lemma \ref{lem:resilience-sq}).  We then show that the implication works in the reverse direction as well: if a class does not contain approximately resilient functions, then it can be agnostically learned by SQ algorithms.  We prove this equivalence using the duality between approximate resilience and approximation by low-degree polynomials stated in Theorem \ref{thm:duality}. This simple observation turns out to be surprisingly useful, leading both to a characterization of agnostic learning and to a proof of our  first structural result for monotone functions (Theorem \ref{talagrand-resilient}).

\ignore{
The learning algorithm for classes which are not approximate resilient is the $\ell_1$-minimization algorithm from \cite{KKMS:08} which agnostically learns any class that is approximated by low-degree polynomials.  We give a formal statement of the result in \cite{KKMS:08} below.  While the $\ell_1$-minimization algorithm given in \cite{KKMS:08} is not a statistcal query algorithm, it is known that it can be implemented as one; we can also rely on the agnostic boosting results of \cite{Feldman:10ab,kanade2009potential}, which do give SQ algorithms for the same classes.
Let $\mathcal{P}_d$ be the class of degree at most $d$ real-valued polynomials.
For a Boolean function $f$, let $\Delta(f,\mathcal{P}_d) = \min_{p\in \mathcal{P}_d} \E[ |f-p|].$
\begin{theorem}[\cite{KKMS:08,Feldman:10ab,kanade2009potential}] \label{thm:kkms+boost}
Fix $0<\epsilon<1$ and $d(n)$ so that every $f\in \mathcal{C}$ satisfies
$\Delta(f,\mathcal{P}_d)\leq 1-\epsilon$.  Then there is an SQ algorithm which runs in time $\poly(n^d,1/\delta)$ and
agnostically learns $\mathcal{C}$ to accuracy $\operatorname{OPT} + \frac{1-\epsilon}{2} + \delta$.
\end{theorem}
}

%\subsection{Approximate resilience and hardness of agnostic learning}
To connect our notion of approximate resilience to the hardness of agnostic learning we will use the following standard notion of designs of sets with small overlap. A $(n,k,d)$-design of size $m$ is a collection of sets $S_1,\dots,S_m \subseteq [n]$ such that $|S_i| = k$ and $|S_i \cap S_j| \leq d$ for all $i\neq j$. Let ${\cal M}(n,k,d)$ denote the size of the largest $(n,k,d)$-design. Standard probabilistic/greedy argument implies that
\begin{equation}\label{eqn:design}
{\cal M}(n,k,d) \geq \frac{\binom{n}{k}}{\binom{k}{d}\binom{n-d}{k-d}} = \frac{\binom{n}{d}}{\binom{k}{d}^2} \geq
\left(\frac{nd}{e^2k^2}\right)^d.
\end{equation}
For a function  $f:\bits^k \to \bits$ and set $S \subseteq [n]$ of size $k$ we use $f_S:\bits^n \to \bits$ to denote $f(\x_{|S})$ where $\x_{|S}$ refers to the restriction of $\x$ to coordinates with indices in $S$ (in the usual order).
\begin{lemma}\label{lem:resilience-sq}
Let $f:\bits^k \to \bits$ be an $\alpha$-approximately $d$-resilient function. Let $S_1,\dots,S_m$ be a $(n,k,d)$-design. If $\{f_{S_i}\}_{i=1}^m \subseteq \C$, then any SQ algorithm for agnostically learning $\C$ with excess error of at most $\frac{1-\alpha}{2}-m^{-1/3}$ has complexity of at least $m^{1/3}$.
\end{lemma}

%==== To be moved to app

To prove Lemma \ref{lem:resilience-sq}, we will use
the following result implicit in \cite{Feldman:12jcss} that is a simple generalization of the well-known SQ-DIM bounds from \cite{blum1994weakly} and their strengthening in \cite{Yang:05,szorenyi2009characterizing}.

% and later strengthened by \cite{Yang:05}. We rely on the simplified prove given by Sz\"or\'enyi in \cite{} which generalizes to concept classes that are real-valued and bounded.
\begin{theorem}\label{thm:sq-lower-bound}
Let $D$ be a distribution and let $g_1,\dots,g_m$ be bounded real-valued functions such that $|\langle g_i, g_j \rangle_D |\leq 1/m$ for $i\neq j$, where $\langle g_i, g_j \rangle_D  = \E_D[g_i(\x) \cdot g_j(\x)]$.
Then any SQ algorithm that for every $i$, given access to statistical queries with respect to distribution $(D,g_i)$ outputs a hypothesis $h$ such that $\E_{(\bx,\boldsymbol{y}) \sim (D,g_i)}[h(\x) \neq \y] \leq \frac{1}{2}- \frac{1}{m^{1/3}}$ has complexity of at least $m^{1/3}$.
\end{theorem}
We can now prove Lemma \ref{lem:resilience-sq}.
\begin{proof}
%Suppose $A$ is an SQ algorithm for agnostic learning $\mathcal{C}$.
By our assumption, the function $f$ is $\alpha$-close to a $d$-resilient bounded function $g:\bits^k \to [-1,1]$. We first note that each pair of functions $g_{S_i}$ $g_{S_j}$ shares at most $d$ relevant variables. These functions are $d$-resilient and therefore there is no single set $T$ such that $\widehat{g_{S_i}}(T) \cdot \widehat{g_{S_j}}(T) \neq 0$. This, by linearity of expectation implies that
 for $i\neq j$, $\E[g_{S_i} g_{S_j}] =0$.

Let $A$ be an agnostic algorithm for $\C$ with excess error of at most $\frac{1-\alpha}{2} -m^{-1/3}$. For every $i$, $f_{S_i}$ is $\alpha$-close to $g_{S_i}$. Therefore if the input distribution is $(U,g_i)$ then $\OPT_\C(U,g_i) \leq \|f_{S_i} -g_{S_i} \|_1/2 = \|f -g \|_1/2 \leq \alpha/2$.
This implies that $A$ will output a hypothesis $h$ with error of at most
$\alpha/2 +\frac{1-\alpha}{2} -m^{-1/3}  =  1/2 -m^{-1/3}$. By Theorem \ref{thm:sq-lower-bound} and orthogonality of $g_{S_i}s$ we get that the complexity of $A$ is at least $m^{1/3}$.
\end{proof}

\ignore{ % VIT: We can use this text for proceedings version
Lemma \ref{lem:resilience-sq} can be proven by applying known SQ bounds (see Theorem \ref{thm:sq-lower-bound}) for families of orthogonal functions to the $d$-resilient functions which are close to $f_{S_1},\dots,f_{S_m}$.  Since an agnostic SQ learning algorithm for $\C$ yields a weak learning algorithm for the set of resilient functions, no such algorithm can exist.  We give the proof of Lemma \ref{lem:resilience-sq} in Appendix \ref{sec:learn-append}.
}

An immediate corollary of Lemma \ref{lem:resilience-sq} is the following lower bound that generalizes Theorem \ref{thm:lb-ar}.
\begin{theorem}\label{thm:lb-ar-general}
Let $\C$ be a concept class closed under renaming of variables and assume that $\C$ contains an $\alpha$-approximately $d$-resilient $k$-junta. Then any SQ algorithm for agnostically learning $\C$ with excess error of at most $\frac{1- \alpha}{2} -m^{-1/3}$ has complexity of at least $m^{1/3}$, where $m = {\cal M}(n,k,d)$. In particular, for any constant $\delta>0$ and $k = n^{1/2+\delta}$, we have $m = n^{\Omega(d)}$.
\end{theorem}

To show that Theorem \ref{thm:lb-ar-general} is essentially tight we prove the duality stated in Theorem \ref{thm:duality} (which we restate here for convenience).
\begin{theorem*}\label{thm:duality-2}[Thm.~\ref{thm:duality} restated]
For $f:\bits^n \to \bits$ and $0\leq d\leq n$ let $\alpha$ denote the $\ell_1$ distance of $f$ to the closest $d$-resilient bounded function. Then $\Delta_{\cP_d}(f) = 1-\alpha$.
\end{theorem*}
\begin{proof}
Our proof is an adaptation of the general results on duality of norms \cite{IoffeTihkomirov:68} to the case where $f$ is Boolean and $g$ is bounded.
In this case it is easy to see that $\|f - g\|_1 = 1-\E[fg]$ and therefore minimization of distance to resilience can be expressed as maximization of $\sum_x f(x) g(x)$ subject to resilience constraints on $g$. Viewing values of $g(x)$ as variables we get:
\begin{alignat*}{2}
\text{max } & \sum_x f(x) g(x) \quad \quad & \\
\text{subject to } & \sum_x g(x) \chi_S(x) = 0 \quad \quad & \forall |S|\leq d \\
\text{and } & |g(x)| \leq 1 &  \quad \quad \forall x\in \bits^n \\
\end{alignat*}
The dual LP can be easily verified to be the following program with variables $p_S$ for every $S \subseteq [n]$ of size at most $d$.
\begin{alignat*}{2}
\text{min } & \sum_x |q(x)|  & \\
\text{subject to } & q(x) = f(x)-\sum_{S:|S|\leq d} p_S \chi_S(x) \quad \quad & \forall x\in \bits^n
\end{alignat*}
Now the claim of the theorem follows from LP duality. By definition the maximum value of the primal is $2^n \cdot \E[fg] = 2^n(1-\|f-g\|_1) = 2^n(1-\alpha)$. This is therefore also the minimum of the dual program which, by definition, is exactly $2^n \cdot \Delta_{\cP_d}(f)$.
\end{proof}

Note that $(1-\alpha)/2$ in the excess error term in the statement of Theorem \ref{thm:lb-ar-general} is equal to $\Delta_{\cP_d}(\C)/2$ in the excess error term in the statement Theorem \ref{thm:kkms+boost}. Therefore combining the duality with the upper-bounds on polynomial $\ell_1$ regression stated in Theorem \ref{thm:kkms+boost} we get our claimed characterization of the complexity of agnostic learning in terms of $\Delta_{\cP_d}(\C)$ or, alternatively, distance to $d$-resilience.

\ignore{
\begin{theorem}\label{thm:char}
Suppose $\mathcal{C}$ is closed under addition of irrelevant variables. Fix $d(n)$ and let $\Delta_{\mathcal{C}_n,d} = \max_{f\in \mathcal{C}_n} \Delta(f,\mathcal{P}_d)$. Then:
\begin{enumerate}
\item
The $\ell_1$ minimization algorithm agnostically learns $\mathcal{C}$ in time $\poly(n^d,1/\delta)$ with error at most $$\operatorname{OPT}+ \frac{\Delta_{\mathcal{C}_n,d}}{2}  + \delta.$$
\item Let $d' = d(n^{1/3})$.  Any algorithm for agnostically learning $\mathcal{C}$ with statistical query complexity at most
$n^{\Omega(d')}$ has error at least
$$\operatorname{OPT} + \frac{\Delta_{\mathcal{C}_n,d'}}{2} - n^{-\Omega(d')}.$$
%$n^{-\Omega(d')}
%\frac{\operatorname{OPT_{adv}}}{\Delta_{\mathcal{C}_{n,d'}}}$.
\end{enumerate}
\end{theorem}
\noindent
To prove Theorem \ref{thm:char}, note that
the learning algorithm from Theorem \ref{thm:kkms+boost} immediately gives (1).  For (2), we will use the following lemma:

}

%% file: tribes.tex
% !TEX root =main2.tex
\section{Monotonicity and approximate resilience}
In this section we prove bounds on the approximate resilience of monotone functions.  First, we give a bound for general monotone functions (Theorem \ref{talagrand-resilient}) in Section \ref{sec:monotone-resilience}.  In Sections \ref{sec:tribes} and \ref{sec:cycle} we show that $\Tribes$ and $\CycleRun$ are approximately resilient (Theorems \ref{tribes-resilient} and \ref{cycle-run-resilient}).
Finally, in Section \ref{sec:amplify} we show how these functions can be used in an iterated construction to yield explicit functions with high approximate resilience.

\subsection{A monotone function with nearly-optimal approximate resilience}\label{sec:monotone-resilience}
Our characterization suggests an approach for proving Theorem \ref{talagrand-resilient}:
since the $\ell_1$-minimization algorithm characterizes SQ agnostic learning, we seek monotone functions where the $\ell_1$-minimization algorithm will badly fail.
In other words, our first step will be to move to the dual problem: Theorem \ref{thm:duality} tells us that we may equivalently show the existence of a monotone function $f$ which is far from from every low-degree polynomial $p$. Strangely, to show that no dual solution exists, we will use the fact that if every monotone function had a weak approximation by some low-degree polynomial,
then the $\ell_1$-minimization algorithm would learn monotone functions, contradicting known information-theoretic lower bounds~\cite{BlumBL:98}.
Note that while the $\ell_1$-minimization algorithm is presented as an agnostic learning algorithm, we may apply it directly to the class of monotone functions.

We now prove Theorem \ref{talagrand-resilient}:
\begin{theorem*}
For every $\alpha > 0$, there is a monotone function that is $\alpha$-approximately $d$-resilient for $d=\Omega(\alpha\sqrt{n}/\log n)$.
\end{theorem*}
\begin{proof}
We show the existence of a monotone function $f$ such that $\E[ |f(\bx)-p(\bx)|] > 1-\alpha$ for every degree-$d$ polynomial $p$ and then apply Theorem \ref{thm:duality}. Suppose that every monotone $f$ satisfies $\E[|f(\bx)-p(\bx)|]\leq 1-\alpha$. Then for $\eps = \alpha/4$, Theorem \ref{thm:kkms+boost} gives an algorithm for learning monotone functions which uses $s=\poly(n^d/\alpha)$ examples and has error $1/2-\alpha/2 + \alpha/4 = 1/2-\alpha/4$. We now use an information-theoretic lower bound on the number of random examples needed to weakly learn monotone functions; the proof in \cite{BlumBL:98} uses a randomized construction of DNF formulas:
\begin{theorem}[\cite{BlumBL:98}] \label{thm:bbl}
Let $A$ be a any learning algorithm that uses $s$ random examples and outputs a hypothesis $h$.  Then there is some monotone $f:\bits^n\to\bits$ such that
$$\Pr[ f(\bx) = h(\bx)] \leq \frac{1}{2} + O\left(\frac{\log sn}{\sqrt{n}}\right).$$
\end{theorem}
Theorem \ref{thm:bbl} tells us that $\alpha = O\left( \frac{d\log n + \log 1/\alpha}{\sqrt{n}}\right),$ which completes the proof.
\end{proof}

The function from Theorem \ref{talagrand-resilient} gives us a $k$-junta that is $\alpha$-approximately $d$-resilient for $d=\Omega(\alpha \sqrt{k}/\log k)$. Plugging this into Theorem \ref{thm:lb-ar-general} and using eq.(\ref{eqn:design}) (assuming $k \leq n^{1/2}$) we obtain the proof of Corollary \ref{sq-lower-bound}.

While the degree of resilience in Theorem \ref{talagrand-resilient} is nearly optimal, the proof is non-constructive and relies crucially on the fact that monotone functions can have high complexity. In the following sections we show that even simple, explicit monotone functions can exhibit high approximate resilience.
\subsection{$\Tribes$ is approximately resilient}\label{sec:tribes}
%TODO: a bit of narrative.
The $\Tribes_{w,s} : \{-1,1\}^{sw} \into \{-1,1\}$ function is the disjunction of $s$ disjoint monotone conjunctions, each of width $w$; i.e.~a read-once width-$w$ DNF.
For notational brevity we write $\Tribes$ to denote $\Tribes_{w,s}$ with $s = (\ln 2)2^w$ (so $w \approx \log n - \log \ln n$ and $s \approx n/(\log n)$).
%Each Fourier coefficient $T\subseteq [n]$ can be easily computed explicitly \cite{Mansour:95,o2007analysis}, see Appendix \ref{app:tribes}.

Our construction of a highly resilient function close to $\Tribes$ is based on a general result relating the low-degree Fourier weight of a Boolean function and its approximate resilience. 
\begin{theorem}
\label{low-degree-resilient}
There exists a universal $K>0$ such that the following holds. Let $f:\{-1,1\}^n \to\{-1,1\}$ be a Boolean function that satisfies $\sum_{|S| \le d} \hat{f}(S)^2 \le \gamma$ for some $d \in [n]$ and $\gamma \in [0,1]$. Then for all $\tau > e^d \sqrt{\gamma}$, we have that $f$ is $O(\tau + \delta n^{2d+2})$-approximately $d$-resilient, where $\delta = \exp\big(-K(\tau^2/\gamma)^{1/d} \big)$.
\end{theorem}

We now prove Theorem~\ref{low-degree-resilient}, and in Section~\ref{sec:set-parameters} we show how Theorem~\ref{tribes-resilient} (i.e.~the approximate resilience of $\Tribes$) follows as a consequence of Theorem~\ref{low-degree-resilient}.

We begin our construction with the Fourier polynomial for $f$ and discard the low-degree terms. That we may do so and hope to arrive at a bounded, resilient function comes from hypercontractivity: since the discarded polynomial has low-degree, it will by highly concentrated around its mean.
The following Chernoff-type concentration inequality for low-degree polynomials over independent Rademacher random variables follows from the hypercontractivity inequalities of Bonami and Beckner \cite{Bon70,Bec75} (see for example \cite{o2007analysis}).
\begin{theorem}[concentration of degree-$d$ polynomials]\label{thm:deg-d-chernoff}
There exists a universal constant $K > 0$ such that for every degree-$d$ polynomial $\{-1,1\}^n \into \R$ and $t > e^d$, we have

\[ \Prx_{\bx}[ | p(\bx) | \geq t\cdot \| p\|_2] \leq \exp\left(-Kt^{2/d}\right).\]
\end{theorem}
%\subsection{Proof of Theorem \ref{tribes-resilient} }
%\begin{lemma}
%There exists a $d$-resilient function $p : \{-1,1\}^n \into [-1,1]$ where
%\begin{enumerate}
%\item for all $|S| \leq d$, $\widehat{p}(S) = 0$,  and
%\item $\Ex_{\bx}[p(\bx) \cdot \Tribes_n(\bx)] \geq 1 - o(1)$.
%\end{enumerate}
%\end{lemma}
%\begin{proof}

\medskip

We now begin the proof of Theorem \ref{low-degree-resilient}.  Let
\[ \ell(x) = \sum_{|S| \leq d} \widehat{f}(S)\chi_{S}(x), \quad \text{and} \quad
h(x) = f(x) - \ell(x).
\]
Our final resilient, bounded function $p$ will be based on $h$, the high-degree part of $f$. Note that while $h$ is $d$-resilient by definition, it may not be uniformly bounded. However, the degree-$d$ Chernoff bound applied to $\ell$ (the low-degree part), together with our assumption on the variance of $\ell$ (i.e.~the low-degree Fourier weight of  $f$), tell us that $\ell$ does not attain large values very often.  Therefore, while $h$ may not be uniformly bounded, we have that $h$ is bounded on almost all inputs $x$ since $h(x) + \ell(x) = f(x) \in \{-1,1\}$.

More formally, we set  $t = \tau/\sqrt{\gamma}$ in Theorem \ref{thm:deg-d-chernoff} (since $\tau > e^d \sqrt{\gamma}$, we have that indeed $t > e^d$)
\[
\Prx_{\bx}[|\ell(\bx)| \geq \tau] \leq \exp \big(-K(\tau^2/\gamma)^{1/d} \big) :=\delta.
\]
%\[
%\Pr_{\bx}[|\ell(\bx)| \geq t\| \ell \|_2] \leq \exp(-Kt^{2/d})
%\]
%We will define $\delta = \exp \left( \dfrac{-K (\tau\sqrt{n})^{2/d}}{(2 \ln n)^3} \right)$, so we can write
%\[
%\Pr_{\bx}[|\ell(\bx)| \geq \tau] \leq \delta
%\]
Next, we define $q : \{-1,1\}^n \into \R$ to be such that
\[
q(x) =
\begin{cases}
0 & \mathrm{if \;} |\ell(x)| > \tau \\
h(x) & \mathrm{if \;} |\ell(x)| \leq \tau.\\
\end{cases}
\]
Since $h(x) = f(x) - \ell(x)$ and $f$ is $\{-1,1\}$-valued, the range of $q$ is $[-1-\tau,1+\tau]$.  While $q$ is bounded, it may now have correlations with low-degree terms (i.e.~$q$ is no longer resilient like $h$ is).  However, we may also write $q$ as $q(x) = h(x) - h(x)\cdot \textbf{1}_{[\ell>\tau]}(x)$, where $h$ is $d$-resilient and $\textbf{1}_{[\ell>\tau]}$ has very small support.  Thus, we will show that we may discard the low-degree terms of $q$ and the effect on boundedness will be uniformly small.
%However, we may now discard the low-degree terms of $q$ and the effect on boundedness will be uniformly small.

Let $q_{>d}(x) = \sum_{|S| \geq d+1} \widehat{q}(S) \chi_{S}(x)$, $q_{\leq d} = q-q_{>d}$ and $p(x) = \frac{q_{>d}(x)}{\|q_{>d}\|_\infty}$.  Certainly, the range of $p$ is $[-1,1]$;  it remains to bound the correlation of $p$ with $f$.  We have that:
\begin{align}
\E[  p\cdot f] & = \E \left[ \frac{(q - q_{\leq d})}{\|q_{>d}\|_\infty}\cdot f\right] \notag \\
& \geq \frac{1}{\|q\|_\infty+\|q_{\leq d}\|_\infty} \cdot \left(\E[ q\cdot f] - \|q_{\leq d}\|_\infty\right) \label{eqn:corr}
\end{align}
%Thus, to bound the correlation of $\Tribes$ with $p$, we will bound its correlation with $q$ and give a bound on $||q-r||_\infty$.
%we first
%bound the correlation with $q$, show that the correlation with $r$ is not much worse, and finally show that the correlation with $p$ is not much worse than that.  The correlation with $q$ is at least
The correlation of $f$ with $q$ is large:
\begin{equation}\label{eqn:corr2}
\Ex_{\bx}[q(\bx) \cdot f(\bx)] \geq (1 - \tau)(1 - \delta) \geq 1 - \tau - \delta.
\end{equation}
The above holds because
the contribution to the correlation is $0$ when $q(x) = 0$, which happens on at most a $\delta$ fraction of the inputs.  On the remaining inputs, $q(x) = h(x) = f(x) - \ell(x)$, and we assumed $|\ell(x)| \leq \tau$.  Thus the contribution on such $x$ is
\[
q(x) \cdot f(x) = (f(x) - \ell(x)) \cdot f(x) = 1 - \ell(x) \cdot f(x) \geq 1 - |\ell(x)| \geq 1 - \tau.
\]
Thus, it only remains to bound the maximum value of the low-degree part of $q$:
\begin{claim}
\[
\displaystyle \|q_{\leq d}\|_\infty \leq \delta n^{2d+2}
%\sum_{|S| \leq d} |\widehat{q}(S)| \leq 2\delta n^{3d/2+1} (2 \ln n)^{d+2}
\]
\end{claim}

\begin{proof}
We will show that $|\hat{q}(S)|< \delta n^{d+1}$ holds for any $|S|\leq d$.
Recalling that $q(x) = h(x)-\textbf{1}_{|\ell|>\tau}\cdot h(x)$, we have:
\begin{align*}
\hat{q}(S) & = \hat{h}(S) - \widehat{\textbf{1}_{|\ell|>\tau} \cdot h}(S)\\
|\hat{q}(S)| & \leq |\hat{h}(S)| + \E[|\textbf{1}_{|\ell|>\tau} \cdot h|]\\
& \leq  0 + \delta\cdot \|h\|_\infty\\
& \leq \delta (\|\ell\|_\infty + 1),
\end{align*}
where the second inequality holds when $|S|\leq d$ because $h$ is $d$-resilient, and the last inequality holds because $|h(x)| \leq |\ell(x)| +1$ for all $x$. 
As $f$ is a Boolean function, each of the non-zero Fourier coefficients of $\ell$ is at most 1 in magnitude. 
%$\Tribes$ is a Boolean function and therefore each of the non-zero Fourier coefficients of $\ell$ is at most 1 in magnitude. 
The rough bound of $n^{d+1}$ on the number of non-zero coefficients of $\ell$ gives a bound of $n^{d+1}$ on $\|\ell\|_{\infty}$; summing over at most $n^{d+1}$ terms of degree at most $d$ gives the claim.
\end{proof}

Let $\kappa = \delta n^{2d+2}$.  Substituting into Equations (\ref{eqn:corr}) and (\ref{eqn:corr2}), we have that
\[
\Ex_{\bx}[p(\bx) \cdot \Tribes(\bx)] \geq \frac{1 - \tau - \delta-\kappa}{1+\tau+\kappa}
\geq 1-\delta-2\tau -2\kappa,
\]
using the fact that $1/(1+x) \geq 1-x$ for $x \geq 0$, and this completes the proof of Theorem~\ref{low-degree-resilient}.

\subsubsection{Proof of Theorem \ref{tribes-resilient}}
\label{sec:set-parameters}

To apply Theorem~\ref{low-degree-resilient} we will need the following upper bound on  the low-degree Fourier weight of $\Tribes$, whose proof is given in Appendix \ref{app:tribes}, can be obtained using the explicit values of each Fourier coefficient given in \cite{Mansour:95},
\begin{proposition}\label{prop:tribes-conc}  For any $d\leq w$ the Fourier weight of $\Tribes$ on degree $d$ and below is at most
\[
\sum_{|S| \leq d} \widehat{\Tribes}(S)^2 \leq 2 \dfrac{(2 \ln n)^{2d+4}}{n}.
\]
\end{proposition}

To derive Theorem~\ref{tribes-resilient} from Theorem~\ref{low-degree-resilient}, we set $\tau = (2 \ln n)^{3d}n^{-2/5}$, so that
$t := \tau /\sqrt{\gamma} \geq n^{1/10}$. Now there exists a small constant $c>0$ such that for $d = c \log n/\log \log n$ and large enough $n$, we have
that $\tau = O(n^{-1/3})$, $t > e^d$ and $t^{2/d} \geq n^{1/(5d)} \geq \frac{3}{K}(\log n)^2 \geq \frac{(2d+3)}{K} \ln n$. This implies that $\delta := \exp \left( -K t^{2/d}\right) \leq n^{-2d-3}$ and so $\delta n^{2d+2} \leq 1/n$. We conclude that $\Tribes$ is $\alpha$-approximately $d$-resilient where $\alpha = O(\tau + n^{-1}) = O(n^{-1/3})$, and this completes the proof of Theorem~\ref{tribes-resilient}.

%\[
%\delta = \exp \left( -K n^{2/d} (2 \ln n)^3 \right) \qquad \mathrm{and} \qquad \kappa = 2 n^{3d/2+1} (2 \ln n)^{d+2} \exp \left( -K n^{2/d} (2 \ln n)^3 \right).
%\]

%When $d = \Theta(\log n/\log \log n)$, we have $t > e^d$, $\tau = n^{-\Omega(1)}$, and $\kappa = \delta n^{2d+2} = O(\tau)$, so
%%%%Therefore
%%%%\[
%%%%\Ex_{\bx}[p(\bx) \cdot \Tribes(\bx)] \geq
%%%%1 - \delta- 2\tau - 2\kappa \geq 1 - O(n^{-1/3}),
%%%%\]
%%%%which completes the proof of Theorem~\ref{tribes-resilient}.
 %Taking $d = \dfrac{\log n}{(\log \log n)^2}$ gives the claim in the title.
%\end{proof}
%\subsection{$\CycleRun$ is $(O(\sqrt{(\log n)/n}), 1)$-approximately resilient}\label{sec:cycle}
\subsection{$\CycleRun$ is approximately resilient: Proof of Theorem \ref{cycle-run-resilient}}\label{sec:cycle}
%\subsection{$\CycleRun$ is approximately resilient}\label{sec:cycle}

\newcommand{\Shift}{\mathrm{Shift}}
\newcommand{\mySum}{\sigma}
\newcommand{\heavy}{\mathrm{heavy}}
\newcommand{\SetS}{\overline{S}}
\newcommand{\Const}{\mathrm{Const}}
\newcommand{\digSum}{\mathrm{DigSum}}
\newcommand{\const}{c}
\newcommand{\floor}[1]{\lfloor #1 \rfloor}
\newcommand{\ceil}[1]{\lceil #1 \rceil}
%\renewcommand{\bx}{\mathbf{x}}
%We prove Theorem \ref{cycle-run-resilient} in this section.  In Section \ref{sec:cycle-run-1-resilient} we show that $\CycleRun$ is $(O(\sqrt{\log n/n}),1)$-approximate resilient, and in Section \ref{sec:amplify} we prove a general amplification lemma that will allow us to amplify the degree of resilience from $1$ to $2^{\widetilde{\Omega}(\sqrt{\log n})}$.

\begin{definition}
\label{def:cyclerun}
For every $n$, the $\CycleRun$ Boolean function $\CycleRun : \{-1,1\}^n \to \{-1, 1\}$ is defined as follows:
Call a consecutive sequence of $1$'s a $1$-run.  Similarly, a consecutive sequence of $-1$'s is a $-1$-run.
We allow runs to wrap around, so if a run reaches $x_n$ it may continue with $x_1$.  The value of $\CycleRun$
is the winner ($1$ for $1$-player or $-1$ for $-1$-player) from the following procedure:
\begin{enumerate}
\item Check which player has the longest run.
\item In case of tie check which player has a larger number of maximum-length runs.
\item In case of tie check the total length of segments between maximum-length runs, where a segment
starting from a $1$-run clockwise is counted for the $1$-player and a segment starting at a $-1$-run clockwise is counted
for the $-1$-player.  The player that has a larger total count is declared the winner.
\end{enumerate}
\end{definition}

We will need that fact that $\CycleRun$ has influence $O(\log n)$. Since the proof of this fact has not appeared in the literature before, we include a proof in Appendix \ref{ap:cycle-run-inf} for completeness.

%\subsection{$\CycleRun$ is $(O(\sqrt{(\log n)/n}), 1)$-approximately resilient}
\label{sec:cycle-run-1-resilient}

%Our first step is to exhibit a pair of close Boolean functions, one that is $1$-resilient and the other one monotone.
\begin{theorem} \label{thm:cyclerun_main}
There exist universal constants $\const_1, \const_2$ such that for every $n \geq \const_2$,
there exists a Boolean function $f : \{-1,1\}^n \to \{-1, 1\}$ such that:
\begin{enumerate}
\item For all $S \subseteq [n]$ such that $|S| \leq 1$, $\widehat{f}(S) = 0$,  and
\item $\Ex_{\bx}[f(\bx) \cdot \CycleRun(\bx)] \geq 1 - \const_1  \sqrt{(\log n)/n}$. %\knote{This might be able to be replaced with $\sqrt{\log(n)/n}$ with more careful analysis; is this possible or not?}
\end{enumerate}
\end{theorem}

Our proof of Theorem \ref{thm:cyclerun_main} relies on four key properties of $\CycleRun$: monotonicity, low influence, oddness, and invariance under cyclic shifts; as far as we know, $\CycleRun$ is the only explicit Boolean function known to have all four properties.  First, as $\CycleRun$ is monotone and transitive, we note that
\[ \widehat{\CycleRun}(\{i\}) = \widehat{\CycleRun}(\{j\}) =  O\left(\frac{\log n}{n}\right)\quad \text{for all $i\neq j \in [n]$}. \]   The high level intuition behind our proof is simple: we show that by flipping the values of $\CycleRun$ from the top of the hypercube downwards and bottom upwards simultaneously, we obtain a balanced function with no Fourier weight at the first level.  This can be done without changing too many points because $\CycleRun$ has small influence;
we are able to do it in a controlled way because it is additionally odd and invariant under cyclic shifts.
%Due to space considerations
We defer the proof of Theorem \ref{thm:cyclerun_main} to Appendix \ref{ap:cycle-run}.

It is natural to wonder how close a monotone function can be to a 1-resilient Boolean function.  We show in Appendix \ref{subsec:mono-res-distance} that Theorem \ref{thm:cyclerun_main} is tight:

\begin{theorem}\label{thm:1-resilient-tight}
For every monotone function $f : \{-1,1\}^n \into \{-1,1\}$ and $1$-resilient $g : \{-1,1\}^n \into \{-1,1\}$, we have $\Pr_{\bx}[ f(\bx) \neq g(\bx)] \geq \Omega\left(\sqrt{\frac{\log n}{n}}\right)$.
\end{theorem}
\subsection{Resilience amplification}
\label{sec:amplify}
\newcommand{\boldb}{\boldsymbol{b}}

In this section we prove a general amplification lemma for resilience.  Given a value $t\in [-1,1]$, we write $\boldb(t)$ to denote a random $\pm 1$ bit with expected value $t$:
\[
\boldb(t) = \left\{
\begin{array}{cl}
1 & \text{with probability $(1+t)/2$} \\
-1 & \text{with probability $(1-t)/2$.}
\end{array}
\right.
\]
 (In particular, $\boldb(1)$ is the constant $1$ and $\boldb(-1)$ is the constant $-1$). Given bounded functions $G : \bits^m \to [-1,1]$ and $g : \bits^n\to [-1,1]$, we define their (disjoint) composition $G \circ g : \bits^{mn} \to [-1,1]$ to be $(G \circ g)(x^1,\ldots,x^m) := \E[G(\boldb(g(x^1)),\ldots,\boldb(g(x^m))]$.  Note that if $\E[g(\bx)] = 0$, then $\E[\boldb(g(\bx))] = 0$ as well.  Throughout this section we write $\dist(f,g)$ to denote $\frac1{2}\E[|f(\bx)-g(\bx)|]$ for notational brevity (this is simply the fractional Hamming distance $\Pr[f(\bx)\ne g(\bx)]$ when $f$ and $g$ are $\{\pm 1\}$-valued).
%For simplicity we prove a version that applies to functions which are close to \textit{Boolean}, resilient functions, and show how this yields an explicit highly approximately resilient monotone function using $\CycleRun$. The extension to bounded functions if fairly straightforward and we omit it in this version. When applied to the $\Tribes$ function this more general version yields a $o(1)$-approximately $2^{\Omega(\sqrt{\log n})}$-resilient function.

%Given Boolean functions $F:\bits^m\to\bits$ and $f:\bits^n\to\bits$, we define their composition $F\circ f : \bits^{mn}\to\bits$ to be $F\circ f := F(f(x^1),\ldots,f(x^m)).$
The main result in this section is the following amplification lemma:
\begin{theorem}
\label{thm:amplification-lemma}
Let $f:\bits^n\to\bits$ and $g:\bits^n \to [-1,1]$ where $\E[f(\bx)] = \E[g(\bx)] = 0$, and suppose $g$ is $d$-resilient. Consider the recursively-defined functions where $f_k = f\circ f_{k-1}$ and $g_k = g\circ g_{k-1}$ for all $k \in \N$, and $f_0 = f$ and $g_0 = g$.  Then for $k\geq 1$:
\begin{enumerate}
\item $f_k$ and $g_k$ are functions over $n^{k+1}$ variables,
\item $g_k$ is $( (d+1)^{k+1} - 1)$-resilient,
\item $\dist(f_k,g_k) \le \dist(f,g) \sum_{t=0}^k \Inf[f]^t$.
\end{enumerate}
\end{theorem}

%
%\begin{theorem}
%\label{thm:amplification-lemma}
%Let $F,G:\bits^m\to\bits$.  Let $f,g:\bits^n\to\bits$ be balanced Boolean functions, and suppose $G$ is $d_1$-resilient and $g$ is $d_2$-resilient.  Consider the recursively-defined functions where $f_k = F\circ f_{k-1}$ and $g_k = G\circ g_{k-1}$ for all $k \in \N$, and $f_0 = f$ and $g_0 = g$.  Then:
%\begin{enumerate}
%\item $f_k$ and $g_k$ are functions over $m^k \cdot n$ variables,
%\item $g_k$ is $( (d_1+1)^k\cdot (d_2+1) - 1)$-resilient,
%\item $\Pr[f_k(\bx) \ne  g_k(\bx)] \le (\Pr[F(\bx)\ne G(\bx)] + \Pr[f(\bx)\ne g(\bx)])\cdot \Inf[F]^k.$
%\end{enumerate}
%\end{theorem}

%The first and second claims are straightforward to verify, and the third follows by noting that the distance between $F\circ f$ and $G\circ g$ can be expressed in terms of the distances between $F$ and $G$, and the noise sensitivity of $F$ at noise rate $\delta$, where $\delta$ is the distance between $f$ and $g$.

The first claim is straightforward to verify, and so we focus on the second and third claims. For a Boolean-valued function $F:\bits^m\to\bits$ and $\delta \in [0,1]$, recall that the \emph{noise-sensitivity of $F$ at noise rate $\delta$} is defined as $\NS_\delta[F] := \Prx_{\y,\z}[F(\y) \ne F(\z)],$ where $\y$ is uniform in $\bits^m$ and $\z$ is obtained from $\y$ by independently flipping each of its coordinates with probability $\delta$.
\label{sec:amplification-lemma}
\begin{lemma}
Given $F,f : \bits^m \to \bits$ and $G,g : \bits^m\to [-1,1]$ where $\E[f(\bx)] = \E[g(\bx)] = 0$, we have
\[ \dist(F\circ f, G \circ g) \le \dist(F,G) + \NS_\delta[F], \]
where $\delta := \dist(f,g)$.
\end{lemma}

\begin{proof}
We first apply the triangle inequality and note that
\[ \dist(F\circ f, G\circ g) \le \dist(F\circ f,  F\circ g) + \dist(F\circ g, G\circ g). \]
Since $\E[g(\bx)] = 0$, we have that $\la \boldb(g(\bx^1)),\ldots,\boldb(g(\bx^m)) \ra$ is uniformly distributed on $\bits^m$ when $\bx^1,\ldots,\bx^m$ are independently and uniformly distributed on $\bits^n$, and therefore the second distance on the right hand side is exactly $ \dist(F,G)$.  Since $\Pr[\boldb(f(x))\ne \boldb(g(x))] = \Pr[f(x) \ne \boldb(g(x))] = \frac1{2} |f(x) - g(x)|$ for all $x \in \bits^n$, it follows that $\Pr[\boldb(f(\bx))\ne \boldb(g(\bx))] = \frac1{2}\E[|f(\bx)-g(\bx)|] = \delta$ and so
\[ \dist(F\circ f, F\circ g) = \Prx_{\y,\z}[F(\y) \ne F(\z)], \]
where $\y$ is uniform in $\bits^m$ and $\z$ is obtained from $\y$ by independently flipping each of its coordinates with probability $\delta$.  This completes the proof, since the probability on the right hand side is precisely $\NS_\delta[F]$.~\end{proof}

Using the union bound, we have

\[
\NS_{\delta}[F] \leq \delta \sum_{i=1}^n \Prx_{\bx}[F(\bx) \neq F(\bx^{\oplus i})] = \delta \cdot \Inf[F] = \dist(f,g) \cdot \Inf[F],
\]
where $\bx^{\oplus i}$ is the string $\bx$ with the $i$-th bit flipped, and $\delta = \dist(f,g)$ as in the previous lemma.  This, along with a straightforward recursion, yields the following corollary.

%Combining the inequality $\NS_\delta[F] \le \delta\cdot \Inf[F]$ (which follows from a standard Fourier-analytic calculation; see e.g.~\cite{o2007analysis}) with a straightforward recursive argument yields the following corollary.

\begin{corollary}
\label{cor:distance-amplification}
Let $f:\bits^n\to\bits$ and $g:\bits^n \to [-1,1]$ where $\E[f(\bx)] = \E[g(\bx)] = 0$, and suppose $g$ is $d$-resilient. Consider the recursively-defined functions where $f_k = f\circ f_{k-1}$ and $g_k = g\circ g_{k-1}$ for all $k \in \N$, and $f_0 = f$ and $g_0 = g$.  Then for $k\geq 1$:
%Given $F,G : \bits^m\to\bits$, and balanced $f,g : \bits^n\to\bits$, we have
%\[ \dist(F\circ f, G \circ g) \le \dist(F,G) + \dist(f,g)\cdot \Inf[F]. \]
%Consequently, if we let $f_0 = f$ and $g_0 = g$ and consider the recursively-defined functions $f_k = F\circ f_{k-1}$ and $g_k = G\circ g_{k-1}$ for all $k\in \N$, we have:
\[ \dist(f_k,g_k) \le \dist(f,g) \sum_{t=0}^k \Inf[f]^t. \]
\end{corollary}

\begin{lemma}
\label{lem:resilience-amplification}
If $G : \bits^m \to [-1,1]$ is $d_1$-resilient and $g: \bits^n \to [-1,1]$ is $d_2$-resilient, then $G\circ g$ is $(d_1d_2)$-resilient.
\end{lemma}

\begin{proof}
By linearity of the Fourier transform it suffices to prove this claim when $G(x_1,\ldots,x_m) = \prod_{i\in T} x_i$ and $|T| > d_1$, the parity function over $d_1 + 1$ or more variables.  We begin by noting that
\begin{eqnarray*} (G \circ g)(x^1,\ldots,x^m) &=& \E\Bigg[\prod_{i\in T}\boldb(g(x^i))\Bigg]  \\
&=& \prod_{i\in T} \E[\boldb(g(x^i))] \\
&=& \prod_{i\in T} \bigg[ \frac{1+ g(x^i)}{2} - \frac{1-g(x^i)}{2} \bigg] \ = \ \prod_{i \in T} g(x^i).
\end{eqnarray*}
We view the $mn$ coordinates of the composed function $G \circ g$ as the disjoint union of $A_1 \cup \cdots \cup A_{m}$, where each $A_i$ has size $n$. With this notation in hand, every subset $S$ of the $mn$ coordinates may be viewed as the disjoint union $S_1 \cup \cdots \cup S_{m}$, where $A_j \subseteq S_j$ for all $j \in [m]$.  Fix $S = S_1 \cup \cdots \cup S_{m}$ of cardinality at most $d_1 d_2$, and recall that our goal is to show that $\widehat{(G \circ g})(S) = 0$. There exists at least one set $S_j$ where $|S_j | \le d_2$, and we assume without loss of generality that $|S_1| \le d_2$. Since $g$ is $d_2$-resilient (in particular, $\widehat{g}(S_1) = 0$), we see that indeed
\[ \widehat{(G\circ g)}(S) = \E\Bigg[\prod_{i \in T} g(\x^i) \prod_{j\in [m]}\prod_{\ell \in S_j} \x_\ell^j \Bigg] = \prod_{i\in T} \hat{g}(S_i) \prod_{j\notin T}\prod_{\ell \in S_j}\E[\x^j_\ell] = 0, \]
and the proof is complete.
%\[ \E\bigg[\boldb(g(\bx^1))\prod_{i \in S_1} \bx_i^1\bigg] = \E\bigg[\frac{(1+g(\bx^1))}{2} \prod_{i\in S_1}\bx^1_i\bigg] - \E\bigg[\frac{(1-g(\bx^1))}{2} \prod_{i\in S_1}\bx^1_i\bigg] = 0,\]
%and so
%\begin{eqnarray*}
% \widehat{({\sf PAR}_{m} \circ g)}(S) &=& \E\bigg[\boldb(g(\bx^1))\cdots\boldb(g(\bx^{m}))\prod_{j \in [m]}  \prod_{i\in S_j}\bx^j_i\bigg]\\
% &=& \prod_{j\in [m]} \E\bigg[\boldb(g(\bx^j))\prod_{i\in S_j} \bx^j_i\bigg] \ = \ 0.
% \end{eqnarray*}
%This completes the proof.
\end{proof}

Combining Corollary \ref{cor:distance-amplification} and Lemma \ref{lem:resilience-amplification} yields Theorem \ref{thm:amplification-lemma}.

\subsubsection{Amplifying $\Tribes$ and $\CycleRun$}
We now apply Theorem~\ref{thm:amplification-lemma} to $\Tribes$ and $\CycleRun$.

\begin{theorem}
There is an explicit $\alpha$-approximately $d$-resilient monotone Boolean function $F$ where $\alpha = o_n(1)$ and $d = 2^{\Omega(\sqrt{\log n})}$.
\end{theorem}

\begin{proof}
We apply Theorem~\ref{thm:amplification-lemma} with $f$ being $\Tribes$ and $g$ the bounded resilient function that results from applying Theorem~\ref{tribes-resilient}.  Since $\Inf[\Tribes] = \Theta(\log n)$ (see e.g.~\cite{KahnKL:88}), taking $k := c \log n/\log\log n$ where $c > 0$ is a sufficiently small universal constant gives functions $f_k,g_k$ over $N := n^k = 2^{O(\log^2 n/\log\log n)}$ variables, where
\[ \dist(f_k, g_k) = O(\Inf[\Tribes]^{k+1}\cdot  n^{-1/3}) = n^{-\Omega(1)} = o_N(1),\] and $g_k$ is $d$-resilient for
\[ d = \Omega((\log n/\log\log n)^{k+1}) = 2^{\Omega(\sqrt{\log N})}. \]
\end{proof}

Analogous calculations for $\CycleRun$ yield the following:

\begin{reptheorem}{cycle-run-resilient}
There is an explicit $\alpha$-approximately $d$-resilient monotone Boolean function $F$ where $\alpha = o_n(1)$ and $d = 2^{\Omega(\sqrt{\log n}/\log\log n)}$. Furthermore, $F$ is $\alpha$-close to a $d$-resilient function that is Boolean-valued as well.
\end{reptheorem}
\begin{proof}
We apply Theorem \ref{thm:amplification-lemma} with $f$ being $\CycleRun$ and $g$ the Boolean-valued resilient function that results from applying Theorem \ref{thm:cyclerun_main}.  Since $\Inf[\CycleRun] = O(\log n)$ (Theorem~\ref{cyclerun-influence}), we again take $k=c \log n/\log \log n$ where $c > 0$ is a sufficiently small universal constant to get Boolean-valued functions $f_k,g_k$ over $N = 2^{O(\log^2 n /\log \log n)}$ variables, where $\Pr[f_k(\bx)\neq g_k(\bx)] = \dist(f_k,g_k) = n^{-\Omega(1)}=o_N(1),$ and $g_k$ is $d$-resilient for
$d = n^{\Omega(1/\log \log n)} = 2^{O\left(\sqrt{\log N}/\log \log N\right)}$.
\end{proof}

%\pagebreak

%\pagebreak

%Here or new section? We could state the cycle run result for 1-resilience, say a few words about the proof, then put some ampfliciation stuff?

%% file: product.tex
\section{Extension to Product Distributions}
\label{sec:product}
We now outline the extension of our characterization of the SQ complexity of agnostic learning to more general product distributions. Let $X$ be the domain of each individual variable, that is our leaning problem is defined over $X^n$. We will start with symmetric product distributions and let $\Pi$ be a distribution over $X$. Let $\B=\{B_0(x),B_1(x),\ldots\}$ be the basis obtained via Gram-Schmidt orthonormalization on the basis $1,x,x^2,\ldots$ with respect to the inner product $\la f,g \ra_\Pi = \E_\Pi[f(\x)g(\x)]$. By definition we obtain that the polynomial degree of $B_i$ is $i$ (for $i \leq |X|-1$). As special cases this process gives $\{1,  \frac{1-\mu \cdot x}{\sqrt{1-\mu^2}}\}$ basis if $X=\bits$ and $\mu = \E_\Pi [\x]$; Legendre polynomials when $X = [-1,1]$ and $\Pi$ is uniform; and Hermite polynomials when $X = \R$ and $\Pi$ is the Gaussian $N(1,0)$ distribution.

%We will truncate this basis at degree $\ell$ to be chosen later.
For $S \subseteq [n]$ and a function $t:S\to \N$ let  $\Phi_{S,t}(x) = \Pi_{i\in S} x_i^{t(i)}$ and
$\Psi_{S,t}(x) = \Pi_{i\in S} B_{t(i)}(x_i)$. For a finite $X$ we restrict the range of such $t$'s to $[|X|-1]$. Clearly, $\Psi$'s are orthonormal functions relative to the inner product $\la f,g \ra_{\Pi^n}=\E_\Pi[f(\x)g(\x)]$.

We now say that a function $g$ is $d$-resilient relative to $\Pi^n$ if for every $S \subseteq [n]$ of size at most $d$ and any function $t:S\to \N$, $\la g, \Psi_{S,t} \ra_{\Pi^n} = 0$. Note that equivalently this can be defined as $\la g, \Phi_{S,t} \ra_{\Pi^n} = 0$ for all  $S \subseteq [n]$ of size at most $d$ and $t:S\to \N$.

We say that a Boolean $f$ is $\alpha$-approximately $d$-resilient relative to $\Pi^n$ if there exists a $d$-resilient $g: X^n \to [-1,1]$ such that $\E_{\Pi^n}[|f(\x)-g(\x)|] \leq \alpha$. In the following discussion functions are over $X^n$ and all norms and inner products relative to $\Pi^n$.

We now describe generalizations of Theorems \ref{thm:kkms+boost}, \ref{thm:duality} and \ref{thm:lb-ar-general}. Let $\cP_{d,\ell}$ denote the class of polynomials where each monomial has at most $d$ different variables each of degree at most $\ell$; let $\cP_{d}=\cP_{d,\infty}$.
Note that by definition this is the span of $\{\Phi_{S,t}\}_{|S| \leq d, t:S \to [\ell]}$ but is also equal to the span of $\{\Psi_{S,t}\}_{|S| \leq d, t:S \to [\ell]}$. For a function $f$, let $\Delta_{\cP_{d,\ell}}(f) = \min_{p\in \mathcal{P}_{d,\ell} } \E_{\Pi^n}[|f(\x)-p(\x)|]$ and for a concept class $\C$, let $\Delta_{\cP_{d,\ell}}(\C) = \max_{f\in \mathcal{C}} \Delta_{\cP_{d,\ell}}(f)$.

The polynomial $\ell_1$ regression algorithm of Kalai \etal for agnostic learning \cite{KKMS:08} applies to this general setting and gives the following bound.
\begin{theorem}[\cite{KKMS:08}] \label{thm:kkms+boost-product}
Let $\C$ be a concept class over $X^n$ and fix $d$ and $\ell$. There exists a SQ algorithm which for any $\eps > 0$ agnostically learns $\C$ over $\Pi^n$ with excess error $\Delta_{\cP_{d,\ell}}(\C)/2 + \eps$ and has complexity $\poly((n\ell)^d ,1/\eps)$.
\end{theorem}

Our SQ lower bound can be easily seen to generalize to the following statement.
\begin{theorem}\label{thm:lb-ar-general-product}
Let $\C$ be a concept class over $X^n$ closed under renaming of variables and assume that $\C$ contains a $k$-junta which is $\alpha$-approximately $d$-resilient over $\Pi^n$. Then any SQ algorithm for agnostically learning $\C$ over $\Pi^n$ with excess error of at most $\frac{1- \alpha}{2} -m^{-1/3}$ has complexity of at least $m^{1/3}$, where $m = {\cal M}(n,k,d)$. In particular, for any constant $\delta>0$ and $k = n^{1/2+\delta}$, we have $m = n^{\Omega(d)}$.
\end{theorem}

Finally, the duality is also easy to verify in this case.
\begin{theorem}
\label{thm:duality-product}
For $f:X^n \to \bits$ and $0\leq d\leq n$ let $\alpha$ denote the $\ell_1$ distance of $f$ to the closest $d$-resilient bounded function. Then $\Delta_{\cP_d}(f) = 1-\alpha$.
\end{theorem}

Now the upper bound is $(n\ell)^{O(d)}$ with excess error $\Delta_{\cP_{d,\ell}}(\C)/2$ and the lower bound is $n^{\Omega(d)}$ with excess error of $\Delta_{\cP_d}(\C)/2$ (if $k$ is not too large). Therefore tightness depends on how fast $\Delta_{\cP_{d,\ell}}(\C)$ approaches $\Delta_{\cP_{d}}(\C)$ as $\ell$ grows. Note that if $\C$ contains only functions that depend on at most $k$-variables then convergence of $\Delta_{\cP_{d,\ell}}(\C)$ to $\Delta_{\cP_{d}}(\C)$ depends only on $k$ (and not on $n$) and also as long as $\ell = n^{O(1)}$ the bounds are still within a polynomial factor.

\medskip
\noindent {\bf Non-symmetric product distributions.}
Now let the domain be $X_1 \times X_2 \times \cdots \times X_n$ and the product distribution be $\Pi = \Pi_1 \times \Pi_2 \times \cdots \times\Pi_n$. We first note that the upper bound in Thm.~\ref{thm:kkms+boost-product} and the duality hold even if the distribution is not symmetric (that is different variables might have different marginal distributions). Therefore we only need to adapt Thm.~\ref{thm:lb-ar-general-product} to this setting.

Our lower-bound construction requires closed-ness with respect to renaming of variables. That would not suffice if different variables have different marginal distributions. For example $\ell_1$ distance to polynomials clearly depends on the marginal distributions of variables and therefore we can no longer claim that the analogue of $\|f_{S_i} -g_{S_i} \|_1 = \|f -g \|_1$ holds in this setting (as we did in the proof of Lemma \ref{lem:resilience-sq}). Therefore we will need an additional assumption. Let $S$ be the set of variables of the optimal (in terms of distance to $d$-resilience) $k$-junta. We will assume that for every variable $i \in S$, there are many other variables that have the same marginal distribution as variable $i$. Specifically, there exists a set $I_i \subseteq [n]$, such that for $j_1,j_2 \in I_i$, $\Pi_{j_1} = \Pi_{j_2}$ and the size of $I_i$ is at least $s$. In addition, we need $\C$ to be closed under renaming of variables, where a variable that is in $I_i$ is renamed to another variable in $I_i$.

 Now we can construct a family of ordered sets $S_1,\ldots,S_m$ (each of size $k$) such that the intersection of any two sets is at most $d$, and the $i$'th element of each set $S_j$ (recall that we think of $S_j$ as an ordered set) is from $I_i$. This means that $X$ and $\Pi$ restricted to variables in $S_j$ (ordered in the same way as they are in $S_j$) are exactly the same as $X$ and $\Pi$ restricted to variables in $S$. This means that the proof of the lower bound in Lemma \ref{lem:resilience-sq} applies to this setting, as before essentially verbatim. The complexity is now determined by the size of the largest family of sets with the property we described. By the same argument as in eq.(\ref{eqn:design}) there exists a family of size:
$$ \frac{s^k}{{k \choose d} s^{k-d}} = \Omega\left(\left(\frac{sd}{k}\right)^d\right).$$ This family has size $n^{\Omega(d)}$ for $s = n^{\Omega(1)}$ and a large range of parameters $k$ and $d$ (e.g.~ $d = k^{1-\Omega(1)}$).

%% file: append.tex
\section{Bound on the low-degree Fourier weight of $\Tribes$}\label{app:tribes}
The $\Tribes_{w,s} : \{-1,1\}^{sw} \into \{-1,1\}$ function is the disjunction of $s$ disjoint conjunctions, each of width $w$.
%Each Fourier coefficient $T\subseteq [n]$ can be computed explicitly (\todo{cite Ryan's notes}), see Appendix \ref{app:tribes}.
For a set $T\subseteq [n]$ let $T_i$ denote the intersection of $T$ with the variables in the $i$-th conjunction. We use the following expressions proved in \cite{Mansour:95}:
\begin{equation}
\widehat{\Tribes_{w,s}}(T) =
\begin{cases}
2(1-2^{-w})^s - 1 &
T = \emptyset \\
2(-1)^{k+|T|}2^{-kw}(1-2^{-w})^{s-k} &
k = \#\{i : T_i \neq \emptyset \} > 0
\end{cases}
\label{tribes-coefficients}
\end{equation}

Recall that we write $\Tribes$ to denote $\Tribes_{w,s}$ with $s = (\ln 2)2^w$; thus $w \approx \log n- \log n \ln n$ and $s \approx n/(\log n)$.

\begin{proposition}  For any $d\leq w$ the Fourier weight of $\Tribes$ on degree $d$ and below is at most
\[
\sum_{|S| \leq d} \widehat{\Tribes}(S)^2 \leq 2 \dfrac{(2 \ln n)^{2d+4}}{n}.
\]
\end{proposition}
\begin{proof}
The proof follows Ryan O'Donnell's thesis, pages $66-67$ \cite{Odonnell:03thesis}.  Using the calculations above, we have that for any $T \subseteq [n]$ with $k=\#\{i: T_i \neq \emptyset\}:$
\[
\widehat{\Tribes}(T)^2 \leq \left( \dfrac{2 \ln n}{n} \right)^{2k}.
\]
For any $k$, the number of coefficients that have degree at most $d$ and intersect $k$ conjunctions is at most
$$\sum_{j=0}^d \binom{s}{k} \binom{kw}{j} \leq (d+1) s^k (kw+1)^d \leq n^k w^{2d+2}.$$
The last inequality holds because $s\leq n$ and $k\leq d$ (and we assume that $d\leq w$).
Summing over $1 \leq k\leq d,$ we obtain:
\begin{align*}
\sum_{|T| \leq d} \widehat{\Tribes}(T)^2 & \leq \sum_{k=1}^d n^k w^{2d+2} \left( \dfrac{2 \ln n}{n} \right)^{2k} \\
& \leq w^{2d+2} \sum_{k=1}^d \left( \dfrac{(2 \ln n)^2}{n} \right)^k \\
& \leq 2w^{2d+2} \dfrac{(2 \ln n)^2}{n} \\
& \leq 2\dfrac{(2 \ln n)^{2d+4}}{n},
\end{align*}
 where we used $w \leq 2 \ln n$ in the last step.
\end{proof}

%% file: cyclerun.tex
% !TEX root =main2.tex
%\newcommand{\CycleRun}{\mathrm{CR}}
\section{Proofs concerning $\CycleRun$}
\label{ap:cycle-run} 

To aid us in proving properties of $\CycleRun$, we will require several bounds involving Gaussian approximations.  Specifically, we will make use of the functions $f_t : \{-1,1\}^n \into \{-1,0,1\}$ that appear in \cite{OWimmer:09}.  We define $|x| = \sum_{i=1}^n x_i$ for a string $x \in \{-1,1\}^n$.  These functions $f_t$ are defined so that

\[
f_t(x) =
\begin{cases}
1 & \mathrm{if \;} |x| > t\sqrt{n} \\
0 & \mathrm{if \;} -t\sqrt{n} \leq |x| \leq t\sqrt{n} \\
-1 & \mathrm{if \;} |x| < -t\sqrt{n}
\end{cases}
\]

\ignore{
Select $t$ to be the smallest $t$ such that $f_t$ satisfies $\Pr[f_t(\bx) \neq 0] \geq \Pr[(f - g)(\bx) \neq 0] = \Pr[f(\bx) \neq g(\bx)]$.  We then have $K \log n \leq \sum_{i \in [n]} \widehat{f-g}(\{i\}) \leq \sum_{i \in [n]} \widehat{f_t}(\{i\})$.
}

We use three properties (implicitly) appearing in \cite{OWimmer:09} that follow from error estimates for the Central Limit Theorem \cite{Feller}: for large enough $n$ and $\sqrt{\log n}/100 < t < n^{1/10}$, we have

\begin{align}
\phi(t)\sqrt{n}/3 & \leq \Inf(f_t) \leq 3\phi(t)\sqrt{n} \label{eqn:prop1} \\
\phi(t)/(3t) & \leq \Pr_{\bx}[f_t(\bx) \neq 0] \leq 3\phi(t)/t \label{eqn:prop2} \\
\Pr[|\bx| = t] & \leq 4\phi(t)/\sqrt{n} \label{eqn:prop3}.
\end{align}

where $\phi$ is the probability density function of the standard Gaussian distribution: $\phi(u) = \frac{1}{2\pi} \exp(-u^2/2)$; and $\Inf(f_t) = \E_{\bx}[f_t(\bx) \cdot |\bx|] = \sum_{i \in [n]} \widehat{f_t}(\{i\})$.  We note that $\Inf(g) = \E_{\bx}[g(\bx) \cdot |\bx|] = \sum_{i \in [n]} \widehat{g}(\{i\})$ for a monotone Boolean function $g : \{-1,1\}^n \into \{-1,1\}$.

\begin{definition}
For every $x \in \{-1,1\}^n$, define the set $\Shift_x$ to contain the following:

\begin{itemize}
\item $x^\alpha = x_{(1 + \alpha \mod n)}\ldots x_{(n + \alpha \mod n)}$, for $0 \leq \alpha \leq n-1$.
\item $-x^\alpha = -x_{(1 + \alpha \mod n)} \ldots -x_{(n + \alpha \mod n)}$, for $0 \leq \alpha \leq n-1$.
\end{itemize}
\end{definition}             

\ignore{Note that for $x \neq 1^n, -1^n$, $|\Shift_x|$ always divides $2n$.}  Note that $|\Shift_x|$ always divides $2n$, and if the Hamming weight of $x$ is relatively prime to $n$, then $|\Shift_x| = 2n$.  Because $\CycleRun$ is odd and invariant under cyclic shifts, $\CycleRun$ is $1$ on exactly half the points of $\Shift_x$.

\ignore{
\begin{definition}
For every $\bx = x_1\ldots x_n \in \{-1,1\}^n$, define $|\bx| = \sum_{i \in [n]} x_i$.
\end{definition}
}

\begin{reptheorem}{thm:cyclerun_main}
There exist universal constants $\const_1, \const_2$ such that for every $n \geq \const_2$, 
there exists a Boolean function $f : \{-1,1\}^n \to \{-1, 1\}$ such that:
\begin{enumerate}
\item For all $S \subseteq [n]$ such that $|S| \leq 1$, $\widehat{f}(S) = 0$,  and 
\item $\Ex_{\bx}[f(\bx) \cdot \CycleRun(\bx)] \geq 1 - 2\const_1 \cdot \sqrt{\frac{\log(n)}{n}}$, which implies $\Prx_{\bx}[f(\bx) \neq \CycleRun(\bx)] \leq \const_1 \cdot \sqrt{\frac{\log(n)}{n}}$.

\end{enumerate}
\end{reptheorem}

\ignore{
To prove these facts, we will construct three approximators $g_1,g_2,g_3$ to $\CR$.  We will ensure that $g_1,g_2$, and $g_3$ are odd, monotone, and invariant under cyclic shifts.  Further, $\widehat{g_1}(\{i\}) = O(n^{-3/2})$, $\widehat{g_2}(\{i\}) = O(n2^{-n})$, and $\widehat{g_3}(\{i\}) = 0$ for all $i \in [n]$.
We will take $g_3$ to be the function $f$ in the lemma.

Given $\CR : \{-1,1\}^n \to \{-1, 1\}$, and an integer $0 \leq i \leq \floor{n/2}$, define $\mathrm{slice}_i(\CR)$ to be the Boolean function such that 

\[
\mathrm{slice}_i(\CR)(\bx) =
\begin{cases}
-1 & \mathrm{if \;} |\bx| > n - i \\
\CR(\bx) & \mathrm{if \;} i \leq |\bx| \leq n - i \\
1 & \mathrm{if \;} |\bx| < i.\\
\end{cases}
\]

Define $\ell$ to be the largest integer such that $\sum_{i \in [n]} \widehat{\mathrm{slice}_{\ell}(i)} \geq 0$.  Clearly, $\sum_{i \in [n]} \widehat{\mathrm{slice}_{\floor{n/2}}(i)}$ is negative, so $\ell$ is well-defined.  Further, $\sum_{i \in [n]} \widehat{\mathrm{slice}_{\ell+1}(i)} < 0$.  We set $g_1 = \mathrm{slice}_{\ell}$.

Next, we select a collection of strings $S$ such that $|\bx| = n - \ell$
for $\bx \in S$, and $\dfrac{2^{n-1}}{n - 2\ell} \sum_{i \in [n]} \widehat{\mathrm{slice}_{\ell}(i)} - 2n |\bigcup_{\bx \in S} \Shift_\bx| \leq \dfrac{2^{n-1}}{n - 2\ell} \sum_{i \in [n]} \widehat{\mathrm{slice}_{\ell}(i)}$.  Let $g_2$ be the function such that 

\[
g_2(\bx) =
\begin{cases}
g_1(\bx) & \mathrm{if \;} \bx \notin S \\
-g_1(\bx) & \mathrm{if \;} \bx \in S
\end{cases}
\]

By construction, $0 \leq \widehat{g_2}(i) \leq n2^{2-n}$ for $i \in [n]$.  We now arbitrarily select a set $U$ of $2^{2-n}\widehat{g_2}(\{1\})$ strings where $|\bx| = \ceil{n/2}$ for all $\bx \in U$ and $T := \bigcup_{\bx \in U} |\Shift_{\bx}|$ satisfies $|T| = 2^{1-n}\widehat{g_2}(\{1\})$.  We now define $g_3$ such that

\[
g_3(\bx) =
\begin{cases}
g_2(\bx) & \mathrm{if \;} \bx \notin T \\
-g_2(\bx) & \mathrm{if \;} \bx \in T
\end{cases}
\]

and by construction, $\widehat{g_3}(i) = 0$ for all $i \in [n]$.

}

\begin{proof}
Given $\CycleRun : \{-1,1\}^n \to \{-1, 1\}$, we construct a set $\SetS \subseteq \{-1,1\}^n$
using the greedy algorithm $\Const_{\overline{S}}(\CycleRun, n)$ described in Figure~\ref{fig:construction}.  

\begin{figure}
\fbox{\parbox{\columnwidth}{
\center{$\Const_{\overline{S}}(\CycleRun, n)$}
\begin{enumerate}
\item Initialize $\SetS = \emptyset$, $\SetS' = \emptyset$.
\item Initialize $\mySum = 2^n \cdot \sum_{i \in [n]}\widehat{\CycleRun}(\{i\})$.  
%I.e. $\mySum$ is initialized the the
%sum of the first degree fourier coefficients of $\CycleRun$.
\item While $|\SetS| \leq \const_1 \cdot \sqrt{\frac{\log(n)}{n}} \cdot 2^n$, do the following:
%Repeat the following at most $\const \cdot \frac{\log(n)}{2n\sqrt{n \log n}} \cdot 2^n$ times:
\begin{description}
\item{$3a.$} Find some $x$ with maximal value of $|x|$
such that $\CycleRun(x) = 1$ and such that $x \notin \SetS$.
%Let $|x|$ denote $\sum_{i \in [p]} x_i$.
\item{$3b.$} If $\mySum - 2 |\Shift_x| \cdot |x| < 0$, then
%if there exists 
find an $x^*\notin \SetS$ such that $|x^*| = 1$ and $\CycleRun(x^*) = 1$ (if no such $x^*$ exists, exit loop and output ``Fail.'').
Then 
%such that $x^* \notin \SetS$,
set $\SetS := \SetS \cup \Shift_{x^*}$, set $\SetS' = \SetS' \cup \Shift_{x^*}$, and set
$\mySum := \mySum - 4n$. If $\mySum=0$, exit the loop.
%If there is no such $x^*$ or $\mySum = 0$, exit the loop.
\item{$3c.$} If $\mySum - 2 |\Shift_x| \cdot |x| > 0$, set $\SetS := \SetS \cup \Shift_x$ and set
$\mySum := \mySum - 2 |\Shift_x| \cdot |x|$.
\end{description} %$\Const_{\SetS}$
\item Return $\SetS$.
\end{enumerate}
}}
\caption{Algorithm for constructing a set of points $\SetS$ used to define the $1$-resilient function $f$.}
\label{fig:construction}
\end{figure}

\medskip

Given the set $\SetS$ outputted by $\Const_{\overline{S}}(\CycleRun, n)$, the function $f : \{-1,1\}^n \into \{-1,1\}$ is defined in the following way:
\[
f(x) =
\begin{cases}
\CycleRun(x) & \mathrm{if \;} x \notin \SetS \\
-\CycleRun(x) & \mathrm{if \;} x \in \SetS.\\
\end{cases}
\]

Clearly, $\Ex_{\bx}[f(\bx) \cdot \CycleRun(\bx)] \geq 1 - 2\const_1 \cdot \sqrt{\frac{\log(n)}{n}}$, 
since the set $\SetS$ satisfies $|\SetS| \leq \const_1 \cdot \sqrt{\frac{\log(n)}{n}} \cdot 2^n$.
Additionally, $f$ is clearly balanced due to the structure of the set $\Shift_x$ of modified points
in each iteration of $\Const_{\SetS}$ and the fact that $\CycleRun$ is odd.
Thus, it remains to show that $\widehat{f}(S) = 0$ for all $S \subseteq [n]$ such that $|S| \leq 1$.

\begin{claim} \label{claim:invariants}
Consider an execution of $\Const_{\overline{S}}$.  At the end of the $i$-th iteration, 
$1 \leq i \leq \const_1 \cdot \sqrt{\frac{\log(n)}{n}} \cdot 2^n$, if $\Const_{\SetS}$ has not terminated,
let $\SetS^i$ denote the current set of points in $\SetS$,
let $\mySum^i$ denote the current setting of the variable $\mySum$ and
let $f^i$ denote the following Boolean function:
\[
f^i(x) =
\begin{cases}
\CycleRun(\bx) & \mathrm{if \;} \bx \notin \SetS^i \\
-\CycleRun(\bx) & \mathrm{if \;} \bx \in \SetS^i.\\
\end{cases}
\]
Additionally, we define $\SetS^0 = \emptyset$, $\mySum^0 = 2^n \cdot \sum_{i \in [n]}\widehat{\CycleRun}(\{i\})$, and
$f^0 = \CycleRun$.

For every $0 \leq i \leq \const_1 \cdot \frac{\log(n)}{2n\sqrt{n}} \cdot 2^n$ the following invariants hold:
\begin{enumerate}
\item $\widehat{f}^i(\{1\}) = \widehat{f}^i(\{2\}) = \cdots = \widehat{f}^i(\{n\})$.
\item $\mySum^i = 2^n \cdot \sum_{j \in [n]}\widehat{f}^i(\{j\})$.
\item $\mySum^i = 4n w \geq 0$, for some integer $w$.
\end{enumerate}
\end{claim}

\begin{proof}
Proof by induction.

\begin{description}
\item{\textbf{Base Case:}}
The base case follows trivially from the definition of $\CycleRun$ and the definition of $\SetS^0$, $\mySum^0$, $f^0$.
\item{\textbf{Inductive Case:}}
Assume the invariants hold for all $0 \leq j \leq i <  \const_1 \cdot \sqrt{\frac{\log(n)}{n}} \cdot 2^n$, we show that the
invariants must also hold for $i + 1$.

For every $j \in [n]$, let us consider the quantity $2^n \left (\widehat{f}^i(\{j\}) - \widehat{f}^{i+1}(\{j\}) \right )$.
Note that by flipping the value of $f^i$ on 
the points in the set $\Shift_x$, $\widehat{f}^i(\{j\})$ is reduced by exactly $1/2^n \cdot 4 \cdot \frac{|\Shift_x| \cdot 
|x|}{2n}$
 for each $j \in [n]$
and so we have that $\widehat{f}^{i+1}(\{1\}) = \widehat{f}^{i+1}(\{2\}) = \cdots = \widehat{f}^{i+1}(\{n\})$.
Moreover, $2^n \left (\sum_{j \in [n]}\widehat{f}^{i}(\{j\}) - \sum_{j \in [n]}\widehat{f}^{i+1}(\{j\}) \right ) = 
2 |\Shift_x| \cdot |x|$
and so we have that
\begin{eqnarray*}
\mySum^{i+1} &=& \mySum^{i} - 2 |\Shift_x| \cdot |x| \\
&=& 2^n \cdot \sum_{j \in [n]}\widehat{f}^{i}(\{j\}) - 2 |\Shift_x| \cdot |x|  \\
&=& 2^n \cdot \sum_{j \in [n]}\widehat{f}^{i+1}(\{j\}),
\end{eqnarray*}
where the second equality holds by the induction hypothesis.

Finally, since $\mySum^{i+1} = 2^n \cdot \sum_{j \in [n]}\widehat{f}^{i+1}(\{j\})$ and $f^{i+1}$ is an odd $\{-1,1\}$-valued function,
we have that
$\mySum^{i+1} = 4nw$ for some integer $w \geq 0$.

\iffalse
Since by the induction hypothesis we have that $\mySum^i = 4pv$ for some integer $v \geq 0$,
we must have that either $\mySum^{i+1}$ is unchanged or that 
\begin{eqnarray*}
\mySum^{i+1} &=& \mySum^i - 2 |\Shift_x| \cdot |\bx|  \\
&=& 4pv - 2 |\Shift_x| \cdot |\bx|  \\
&=& 4pv - 4nu \\
&=& 4pw, 
\end{eqnarray*}
for some integer $w \geq 0$.
\fi
\end{description}
\end{proof}

\medskip

\ignore{
\begin{claim} \label{claim:heavy}
Let $\overline{S}_{\heavy} \subseteq \overline{S}$ be the subset of points $\bx \in \{-1,1\}^n$ in $\SetS$ such that $|\bx| \geq \sqrt{n \log n}$.
$|\SetS_{\heavy}| \geq |\SetS| - |\SetS'|$.
\end{claim}

\begin{proof}
Note that since $\CycleRun$ is odd, we have that $\sum_{x : |x| \geq 2\sqrt{n}} \CycleRun(x) 
= -\sum_{x : |x| \leq -2\sqrt{n}} \CycleRun(x)$.
Moreover, since $\CycleRun$ is monotone, we must have that $\sum_{x : |x| \geq 2\sqrt{n}} \CycleRun(x) \geq 
\sum_{x : |x| \leq -2\sqrt{n}} \CycleRun(x)$.  Therefore, we must have that 
$\sum_{x : |x| \geq 2\sqrt{n}} \CycleRun(x) \geq 0$.
Since $\CycleRun$ is $\{-1,1\}$-valued, this immediately implies that
at least half of the points $x$ where $|x| \geq 2\sqrt{n}$ are such that $\CycleRun(\bx) = 1$.
Since the number of points $x$ where $|x| \geq 2\sqrt{n} = O(2^n)$, and since $\Const_{\SetS}$ chooses points greedily with
maximal $|\bx|$, we must have that $\overline{S}_\heavy = \SetS \setminus \SetS'$.
\end{proof}

We show that Claim~\ref{claim:correctness}, Claim~\ref{claim:SetSprime} 
and Claim~\ref{claim:heavy} imply that when $\Const_{\SetS}$ completes,
it is always the case that 
%for $\SetS$ 
%returned by $\Const_{\SetS}$, we have that $|\SetS| \leq \frac{\log(p)}{\sqrt{p}} \cdot 2^p$ and
$\mySum = 0$.  Combined with Claim~\ref{claim:invariants} Items~$1$ and $2$, this is sufficient to complete the proof of 
Theorem~\ref{thm:cyclerun_main}.

We first argue that Claim~\ref{claim:heavy} implies that $\Const_{\SetS}$ always reaches a point where the condition in line $3b$ is
true. 
Assume not.  This means that when $\Const_{\SetS}$ terminates,
$|\SetS| = \const_1 \cdot \frac{\log(n)}{\sqrt{n}} \cdot 2^n$ and by
Claim~\ref{claim:heavy}, at least $\const_1 \cdot \frac{\log(n)}{\sqrt{n}} \cdot 2^n - 2n^2$ 
of these points have $||\bx|| \geq \sqrt{n}$.  We know more than this; by our defintion of $c_3$ and using the average value of $||\bx||$ on these points, we have that at the end of the protocol, $\mySum^i < \mySum - 2^n \cdot \const_1 \log(n) \leq 0$, which is a contradiction to
Item~$3$ listed in Claim~\ref{claim:invariants}.
}

\ignore{
We define $\const_3$ so that the average value of $||\bx||$ is greater than $\const_3 \cdot \sqrt{n \log n}$ on the $\frac{\const_1} \cdot 2^n \sqrt{\log n/n}$ points with maximum $|\digSum_{\bx}|$.  This is possible via standard Central Limit Theorem arguments; see Equation~\ref{eqn:prop2} of Section~\ref{subsec:mono-res-distance} with $t = \Theta(\sqrt{\log n})$ with a constant so that $\phi(t) = \Theta(\log n/\sqrt{n})$ for details.
}

\ignore{
\begin{claim} \label{claim:heavy}
Let $\overline{S}_{\heavy} \subseteq \overline{S}$ be the subset of points $x \in \{-1,1\}^n$ in $\SetS$ such that $|\bx| \geq \sqrt{n}$.
Then $|\SetS_{\heavy}| \geq |\SetS| - |\SetS'|$.
\end{claim}

\begin{proof}

Note that since $\CycleRun$ is odd, we have that $\sum_{x : |x| \geq \sqrt{n}} \CycleRun(x) 
= -\sum_{x : |x| \leq -\sqrt{n}} \CycleRun(x)$.
Moreover, since $\CycleRun$ is monotone, we must have that $\sum_{x : |x| \geq \sqrt{n}} \CycleRun(x) \geq 
\sum_{x : |x| \leq -\sqrt{n}} \CycleRun(x)$.  Therefore, we must have that 
$\sum_{x : |x| \geq \sqrt{n}} \CycleRun(x) \geq 0$.
Since $\CycleRun$ is $\{-1,1\}$-valued, this immediately implies that
at least half of the points $x : |x| \geq \sqrt{n}$ are such that $\CycleRun(x) = 1$.
Since the number of points $x$ such that $|x| \geq \sqrt{n}$ is $\Omega(2^n)$, and since $\Const_{\SetS}$ chooses points greedily with
maximal $|x|$, we must have that $\overline{S}_\heavy = \SetS \setminus \SetS'$ and thus $|\SetS_{\heavy}| \geq |\SetS| - |\SetS'|$.
\end{proof}

We show that Claim~\ref{claim:correctness}, Claim~\ref{claim:SetSprime} 
and Claim~\ref{claim:heavy} imply that when $\Const_{\SetS}$ completes,
it is always the case that 
%for $\SetS$ 
%returned by $\Const_{\SetS}$, we have that $|\SetS| \leq \frac{\log(p)}{\sqrt{p}} \cdot 2^p$ and
$\mySum = 0$.  Combined with Claim~\ref{claim:invariants} Items~$1$ and $2$, this is sufficient to complete the proof of Theorem \ref{thm:cyclerun_main}.
}

% We first argue that Claim~\ref{claim:heavy} implies that the algorithm always reaches a point where the condition in line $3b$ is true.  

We proceed to show that $\Const_{\SetS}$ terminates.  Our goal is to show that at the termination of the algorithm, we have $\sigma = 0$.

\begin{claim}
\label{claim:3btrue}
The algorithm $\Const_{\SetS}$ always reaches a point where the condition in line $3b$ is true.
\end{claim}

\begin{proof}
We use the functions $f_t$ from the beginning of this section.  Take $t' = \sqrt{\log n - 2\log \log n - C}$ for a constant $C$ to be determined later.  Then $\phi(t') = \frac{1}{2\pi}e^{C/2}(\log n)/\sqrt{n}$, so $\Inf(f_{t'}) \geq \frac{1}{6\pi}e^{C/2} \log n$ and $\Pr_x[f_{t'}(x) \neq 0] \leq 
\frac{3}{2\pi}e^{C/2}/t' \leq \frac{3}{\pi}e^{C/2} \sqrt{\log n/n}$ by Equations~\ref{eqn:prop1} and~\ref{eqn:prop2} respectively.
We choose $C$ so that $\Inf(f_{t'}) \geq 3 \cdot \Inf(\CycleRun)$, which can be done since $\Inf(\CycleRun) = O(\log n)$.

We claim that $\Const_{\SetS}$ does not include any strings $x$ in $\SetS$ with $3 \leq |x| < t'$ (and thus none with $-t' < |x| \leq -3$).  Suppose that this claim is false.  Because the algorithm is greedy, then every string $x$ where $\CycleRun(x) = 1$ with $t' \leq |x| \leq n$ is corrupted and in $\SetS$.  Since $\CycleRun$ is odd and monotone, at least half of the strings where $|x| = k$ are corrupted for $t' \leq k \leq n$.  The contribution to be reduction in the first-order Fourier coefficients when we flip the value on these strings from $1$ to $-1$ is at least $(1/2)\Inf(f_t') \geq (3/2)\Inf(\CycleRun)$.  But this implies that the sum of first-order Fourier coefficients for the corrupted function is at most $-(1/2)\Inf(\CycleRun) < 0$.  This implies that $\sigma < 0$ in the execution of $\Const_{\SetS}$, which is a contradiction since $\sigma$ stays nonnegative during the execution of the algorithm.

It remains to show that the condition in line $3$ is satisfied throughout the execution of $\Const_{\SetS}$.  Because no strings with $3 \leq |x| < t'$ or $t' < |x| \leq -3$ are corrupted, the fraction of strings corrupted is at most $\Pr_{\bx}[f_{t'}(\bx) \neq 0] + \Pr_{\bx}[|\bx| = \pm1] = O(\sqrt{\log n/n})$.  Thus at most $c_1 \sqrt{\frac{\log n}{n}} 2^n$ strings are in $\SetS$, so the condition in line $3$ holds.

\ignore{
Assume not.  This means that when the algorithm terminates,
$|\SetS| \leq \const_1 \cdot \sqrt{\frac{\log n}{n}} \cdot 2^n$ and by
Claim~\ref{claim:heavy}, at least $\const_1 \cdot \const_3 \cdot \sqrt{\frac{\log n}{n}} \cdot 2^n - 2n^2$ 
of these points have $||\bx|| \geq \sqrt{n}$.  We know more than this, however.  By our definition of $c_3$, we know the average value of $||\bx||$ is large on these points.  We have that at the end of the protocol, $\mySum^i < \mySum - 2^n \cdot \const_1 \log(n) \leq 0$, which is a contradiction to
Item~$3$ listed in Claim~\ref{claim:invariants}.
}

\end{proof}

\ignore{Note that $\E_{\bx}[f(\bx) \cdot \CycleRun(\bx)] \geq 1 - 2c_1 \sqrt{\frac{\log n}{n}}$.}

Next, we argue that when $\Const_{\SetS}$ reaches the point where the condition in line $3b$ evaluates
true, there always exists a point $x^* \notin \SetS$ such that
$\CycleRun(x^*) = 1$ and $|x^*| = 1$.  We first prove two lemmas.

\begin{lemma} \label{lemma:correctness}
Let $S^1_1$ be the set of $x \in \{-1,1\}^n$ such that $|x| = 1$ and $\CycleRun(x) = 1$.
Then $|S^1_1| \geq 2n^2$.
\end{lemma}

\begin{proof}
\ignore{Let $S_z$ denote the set of points such that $|\bx| = z$.}
Note that since $\CycleRun$ is odd, we have that $\sum_{x:|x|=\pm1} \CycleRun(x) = 0$.
Moreover, since $\CycleRun$ is monotone, we must have that $\sum_{x:|x|=1} \CycleRun(x) \geq 
\sum_{x:|x|=-1} \CycleRun(x)$.  Therefore, we must have that $\sum_{x:|x|=1} \CycleRun(x) \geq 0$.
Since $\CycleRun$ is $\{-1,1\}$-valued, this immediately implies that
at least half of the points $x$ where $|x| = 1$ are such that $\CycleRun(x) = 1$.
There are ${n \choose (n-1)/2} \geq 4n^2$ such strings where $|x| = 1$, so we have that $|S^1_1| \geq 2n^2$.
This concludes the proof of Lemma~\ref{lemma:correctness}.
\end{proof}

\begin{lemma} \label{lemma:SetSprime}
$|\SetS'| \leq 2n^2$.
\end{lemma}

\begin{proof}
Consider the first time the condition in line $3b$ evaluates to true.  Then there is some 
some $x$ 
such that $\CycleRun(x) = 1$ and such that $\mySum - 2 |\Shift_x| \cdot |x| < 0$.
Since $|x| \leq n$, this implies that $\mySum \leq 4n^2$.
Moreover, in each iteration $2n$ points are added to $\SetS'$, and $\mySum$ is reduced by $4n$.
Thus, after at most $n$ iterations, $\mySum$ is reduced to $0$.  These iterations are the only iterations that contribute to $\SetS'$, so $|\SetS'| \leq n \cdot 2n = 2n^2$ as claimed.
\end{proof}

We proceed to show that the when the condition in line $3b$ is true,
there is an $x^* \notin \SetS$ such that $\CycleRun(x^*) = 1$ and $|x^*| = 1$.
%Note that we must have  $w < p$ since $\bx \neq 1^p$ and so $p > |\bx| > w$.
By Lemma~\ref{lemma:correctness}, 
there exist at least $2n^2$ number of points $x^*$ such that $\CycleRun(x^*) = 1$ and $|x^*| = 1$.  
Thus, if $\Const_{\SetS}$ reaches a point where the condition in line $3b$ evaluates to true and there is no point
$x^* \notin \SetS$ such that
$\CycleRun(x^*) = 1$ and $|x^*| = 1$, then it
must be the case that all such $x^*$ are already contained in $\SetS$.
But since we have by Lemma~\ref{lemma:SetSprime} that $|\SetS'| \leq 2n^2$ then we must have that some 
point $y$ such that $\CycleRun(y) = 1$ and $|y| = 1$ was added to $\SetS$ before the first
time the condition in line $3b$ evaluates to true.
But the first time the condition in line $3b$ evaluates to true, we must have that $|x| > 1$, and 
since $\Const_{\SetS}$ always chooses to add points $y$ with maximal
$|y| \geq |x| > 1$ to the set $\SetS$, this is impossible.

We have now argued that $\Const_{\SetS}$ always reaches a point where the condition in line $3b$ is
true, and that whenever this occurs there always exists a point $x^* \notin \SetS$ such that
$\CycleRun(x^*) = 1$ and $|x^*| = 1$.  This immediately implies that when
$\Const_{\SetS}$ completes, we have $\mySum = 0$ and $|\SetS| \leq c_1 \sqrt{\frac{\log n}{n}} 2^n$.  As in the beginning of the proof, we take $f$ to be function to be the function such that

\[
f(x) =
\begin{cases}
\CycleRun(x) & \mathrm{if \;} x \notin \SetS \\
-\CycleRun(x) & \mathrm{if \;} x \in \SetS.\\
\end{cases}
\]

Clearly, $\Pr_{\bx}[f(\bx) \neq \CycleRun(\bx)] = |\SetS| \leq c_1 \sqrt{\frac{\log n}{n}} 2^n$, and
applying the invariants of Claim~\ref{claim:invariants} shows that $f$ is $1$-resilient, concluding the proof of Theorem~\ref{thm:cyclerun_main}.
\end{proof}

This analysis almost works for any balanced monotone function with influence $O(\log n)$, such as $\Tribes$.  While the above could be adapted in a straightforward matter to show that there is a Boolean function close to $\Tribes$ with very small constant and first-order Fourier coefficients, showing that all of these Fourier coefficients can be made \emph{exactly} zero seems challenging.  Since we are applying these results to juntas, our proofs can not tolerate even exponentially small Fourier coefficients.  The structure of $\CycleRun$ is quite amenable to ``local'' changes while retaining structure.

%% file: cycleruninf.tex
% !TEX root =main2.tex
\newcommand{\by}{\boldsymbol{y}}
\subsection{Influence bound for Cycle Run}

\label{ap:cycle-run-inf} 

\newcommand{\calU}{\mathcal{U}}
\newcommand{\Geo}{\mathcal{G}}
\newcommand{\bg}{\boldsymbol{g}}

The main result of this section is the following:

\begin{theorem}
\label{cyclerun-influence} 
$\Inf(\CycleRun) = O(\log n)$.
\end{theorem}

The condition on $\CycleRun$ given in Definition~\ref {def:cyclerun} implies that for every influential edge $(x,x^{\oplus i})$, at least one of the endpoints is in the first two cases in Definition~\ref{def:cyclerun}, and the pivotal coordinate $i$ occurs in a maximum length run.  Thus $\Inf(\CycleRun) \leq 2\Ex_{\bx \sim \calU}[\ell(\bx) \cdot (r_{\ell(\bx)}(\bx) + 1)]$, where $\ell(x)$ is the maximum length run in the string $x$, $r_i(x)$ is the number of maximal runs of length exactly $i$ in $x$, and $\calU$ is the uniform distribution on $\{-1,1\}^n$.  In this section, we will not consider the runs wrapping around, and the $+1$ here takes care of the case that we ``split'' the cycle in a maximum length run to lay out the bits in a line.  

We make use of a result from \cite{schilling1990longest}:

\begin{theorem}
$\Ex_{\bx \sim \calU}[\ell(\bx)] = O(\log n)$
\end{theorem}

\newcommand{\bb}{\boldsymbol{b}}

Thus $\Inf(\CycleRun) \leq 2\Ex_{\bx \sim \calU}[\ell(\bx) \cdot r_{\ell(\bx)}(\bx)] + O(\log n)$, so the remainder of the section is devoted to showing $\Ex_{\bx \sim \calU}[\ell(\bx) \cdot r_{\ell(\bx)}(\bx)] = O(\log n)$.  To aid in our analysis, we will consider different distributions over binary strings.  Consider the following method of generating a string $\bx \sim \calU$:

\begin{enumerate}
\item Initialize $\bx$ to the empty string, and set $b$ to a uniform $\pm1$ random bit $\bb$.

\item (Iterative step) Assuming there are still $j > 0$ bits of $\bx$ to determine, then draw $\bg \sim \mathrm{Geometric}(1/2)$ conditioned on $\bg$ being at most $j$, and set the next $\bg$ bits of $\bx$ to $b$.

\item If not all $n$ bits of $\bx$ are set, set $b$ to $-b$ and return to step 2.

\item If all bits of $\bx$ are set, then $\bx$ is a uniformly random string in $\{-1,1\}^n$.
\end{enumerate}

Further, if we want to condition on the maximum run in $\bx$ being at most some value $t$, we can replace the conditioning in step $2$ from ``being at most $j$'' to ``being at most $\min\{t,j\}$''. 

\begin{lemma}
\label{lem:geo-condition}
For $\bg \sim \mathrm{Geometric}(1/2)$, and $1 \leq g \leq t$, we have $\Pr[\bg = g | \bg \leq t] \leq 2\Pr[\bg = g]$.
\end{lemma}

\begin{proof}
Follows directly from conditional probability and the fact that $\Pr[\bg \leq t] \geq 1/2$ for all $t \geq 1$.
\end{proof}

For an integer $k > 0$, we define the distribution $\Geo_k$ on binary strings of varying length such that a draw from $\Geo_k$ is $\bb^{\bg_1}(-\bb)^{\bg_2}\bb^{\bg_3}
\cdots \bb^{\bg_k}$ if $k$ is odd and $\bb^{\bg_1}(-\bb)^{\bg_2}\bb^{\bg_3}
\cdots (-\bb)^{\bg_k}$ if $k$ is even.  Here, the $\bg_i$'s are independent $\mathrm{Geometric}(1/2)$ variables, and $\bb$ is a uniform $\pm1$ bit.

\begin{lemma} \label{lem:numbermaximal}
\[
\Ex_{\bx \sim \calU}[\ell(\bx) \cdot r_{\ell(\bx)}(\bx) | \ell(\bx) = t] \leq t(2^{1-t}n + 1)
\]
\end{lemma}

\begin{proof}

We first claim that
\[
\Ex_{\bx \sim \calU}[\ell(\bx) \cdot r_{\ell(\bx)}(\bx) | \ell(\bx) = t] \leq t + \Ex_{\bx \sim \calU}[\ell(\bx) \cdot r_t(\bx) | \ell(\bx) \leq t]
\]

To see this, note that if we further condition on the first run of length $t$ selected, this expectation is maximized when the first run is of length $t$.  Also, the expectation can only increase if we allow all $n$ more bits to be set rather than $n-t$.  Since the first run is of length $t$, we only need the maximum length run to be at most $t$ in the rest of the string.

Now we have

\[
t + \Ex_{\bx \sim \calU}[\ell(\bx) \cdot r_t(\bx) | \ell(\bx) \leq t]
\leq t + t \Ex_{\bx \sim \calU}[r_t(\bx) | \ell(\bx) \leq t] \leq
t + t \Ex_{\by \sim \Geo_n}[r_t(\by) | \ell(\by) \leq t] 
\]

where the second inequality comes from the fact that $\bx$ is generated by at most $n$ runs, and not bounding the length of the string only increases the possible number of runs of length $t$, conditioned on the maximum length run being at most $t$.  By Lemma~\ref{lem:geo-condition},
the probability of a single run being of length $t$ is at most $2^{1-t}$, so we have

\[
t + t \Ex_{\by \sim \Geo_n}[r_t(\by) | \ell(\by) \leq t] \leq t + t (2^{1-t}n) = t(2^{1-t}n + 1)
\]

completing the proof.
\end{proof}

\ignore{
\begin{lemma}
\[
\Ex_{\bx \sim \calU}[\ell(\bx) \cdot r_{\ell(\bx)}(\bx) | \ell(\bx) = t] \leq 2\Ex_{\by \sim \Geo_n}[\ell(\by) \cdot r_{\ell(\by)}(\by) | \ell(\by) = t] \leq t(2^{1-t}n + 1)
\]
\end{lemma}

\begin{proof}
We can couple a draw from $\bx \sim \calU$ and a draw from $\by \sim \Geo_n$ for the first few runs conditioned on $\ell(\bx) = t$, until more than $n-t$ bits of $\bx$ are fixed.  After this point, $\bx$ can contain no more runs of length $t$, while $\by$ can, so the expected number of such runs can only increase.  

Conditioning on $\ell(\bx) = t$, $\ell(\by) = t$, all runs in $\bx$ or $\by$ are of length at most $t$, so the probability that any run is of length exactly $t$ is at most $2^{1-t}$.  Because the runs in $\by$ are of independent length, we can arbitrarily choose a run to have length exactly $t$, and consider the expected number of runs of length $t$ conditioned on $\ell(\by) \leq t$ for the rest.  Thus the expected number of such runs is $2^{1-t}(n-1) + 1 \leq 2^{1-t}n + 1$, finishing the claim since we fix the maximum length run to be $t$.

\end{proof}

}
\begin{lemma}
\[
\Pr_{\bx \sim \calU}[\ell(\bx) \leq t] \leq (1 - 2^{-t})^{n/8} + \exp(-n/32)
\]
\end{lemma}

\begin{proof}
For $\bx \in \bits^n$, let $\mathrm{runs}(\bx)$ be the number of runs in $\bx$.  
We first show that with probability at least $1 - \exp(-n/32)$, a string $\bx \sim \calU$ has $\mathrm{runs}(\bx)\geq n/8$ .  To do this, we prove that with probability $1-\exp(-n/32)$, the first $n/8$ runs of $\bx$ contain at most $n/2$ bits.
Note that we may instead bound the number of bits in $\by \sim \Geo_{n/8}$, since each run of $\Geo_{n/8}$ can only be longer.  
%the number of bits in the first $n/8$ runs of $\bx \sim \calU$ is at most the number of bits in $\by \sim \Geo_{n/8}$, since each run of $\Geo_{n/8}$ can only be longer.  
%Thus to lower bound the number of runs in $\bx$ it suffices to give an upper bound on the number of bits in $\by$.  

%After $n/8$ runs, 
The expected number of bits in $\Geo_{n/8}$ generated is $n/4$, and this number of bits is concentrated around its mean; the number of bits has a negative binomial distribution.  By \cite{Brown:dj}, we have
\[
\Pr_{\by \sim \Geo_{n/8}}[\mathrm{bits}(\by) > 2(n/4)] \leq \exp(-n/32)
\]
where the second inequality holds because the number of runs does not increase the probability of getting a longer run, and the distributions of the lengths of each run in $\bx$ are identical to (or conditioned on being shorter than) the lengths of the runs in $\Geo_{n/8}$. 
%where $\mathrm{bits}(\by)$ denotes the number of bits in $\by$.
We then have: % Thus with probability at least $1 - \exp(-n/32)$, we have at least $n/8$ runs and $n/2$ bits set.  It follows that
\begin{align*}
\Pr_{\bx \sim \calU}[\ell(\bx) \leq t] & \leq \Pr_{\bx \sim \calU}[\ell(\bx)\leq t,\mathrm{runs}(\bx)\geq n/8] + \exp(-n/32)\\
& \leq \Pr_{\by \sim \Geo_{n/8}}[\ell(\by) \leq t] + \exp(-n/32)
\end{align*}
%where $\mathrm{runs}(\bx)$ is the number of runs in $\bx$, and the second inequality  
where the second inequality holds because the length of each run of $\bx$ is distributed identically (or conditioned to be shorter) to each run of $\by$, and considering fewer runs only decreases the chances of obtaining a run longer than $t.$ 
%decreasing the number of runs will not increase the probability of obtaining a longer run,  
It is then straightforward to calculate $\Pr_{\by \sim \Geo_{n/8}}[\ell(\by) \leq t] = (1 - 2^{-t})^{n/8}$, since $\Pr[\bg \leq t] = 1 - 2^{-t}$ for $\bg \sim \mathrm{Geometric}(1/2)$.
\end{proof}

We now proceed to show $\Ex_{\bx \sim \calU}[\ell(\bx) \cdot r_{\ell(\bx)}] = O(\log n)$, starting by applying total expectation and applying Lemma \ref{lem:numbermaximal}:

\begin{align*}
\Ex_{\bx \sim \calU}[\ell(\bx) \cdot r_{\ell(\bx)}(\bx)] & = 
\sum_{t=1}^n \Pr_{\bx \sim \calU} [\ell(\bx) = t] \E_{\bx \sim \calU} [\ell(\bx) \cdot r_{\ell} | \ell(\bx) = t] \\
& \leq
\sum_{t=1}^n \Pr_{\bx \sim \calU} [\ell(\bx) = t] t(2^{1-t}n + 1) \\
& \leq
\E_{x\sim U}[\ell(\bx)] + \sum_{t=1}^n \Pr_{\bx \sim \calU} [\ell(\bx) = t] t2^{1-t}n \\
& \leq
O(\log n) + \sum_{t=1}^n ((1 - 2^{-t})^{n/8} + \exp(-n/32))t2^{1-t}n \\
& \leq
O(\log n) + \sum_{t=1}^n (1 - 2^{-t})^{n/8})t2^{1-t}n \\
& \leq
O(\log n) + \sum_{t=1}^n tn2^{1-t}\exp(-2^{-t}n/8)\\
\end{align*}

Letting $a_t = tn2^{1-t} \exp(-2^{-t}n/8)$, we see that $a_{t-1}/a_t < 3/4$ when $2 \leq t \leq \log n - 10$, and $a_{t+1}/a_t < 3/4$ when $\log n + 10 \leq t \leq n$.
Also, $a_t \leq O(\log n)$ for each term where $\log n - 10 \leq t \leq \log n + 10$.  So the proof is completed by noting the above is at most

\begin{align*}
O(\log n) + \sum_{t=2}^{\log n - 10} a_{\log n - 10} (3/4)^{\log n - 10 - t} + \sum_{t=\log n-9}^{t=\log n+9} a_t + \sum_{t=\log n+10}^n a_{\log n + 10} (3/4)^{t-(\log n+10)} \\
\leq
O(\log n)\left(\sum_{t=2}^{\log n - 10} (3/4)^{\log n - 10 - t} + \sum_{t=\log n-9}^{t=\log n+9} 1 + \sum_{t=\log n+10}^n (3/4)^{t-(\log n+10)}\right) = O(\log n). \\
\end{align*}

\subsection{Lower bound for monotonicity-resiliency distance}
\label{subsec:mono-res-distance}
We give a lower bound for distance between monotonicity and resiliency that matches the bound for $\CycleRun$ up to constant factors.

\begin{theorem}
For every monotone function $f : \{-1,1\}^n \into \{-1,1\}$ and $1$-resilient $g : \{-1,1\}^n \into \{-1,1\}$, we have $\Pr_{\bx}[ f(\bx) \neq g(\bx)] \geq \Omega(\sqrt{\frac{\log n}{n}})$.
\end{theorem}

\begin{proof}
If $\Var[f]<1/2$, then $\widehat{f}(\emptyset)^2 > 1/2$, and $\Pr[f\neq g] \geq \frac{1}{4} E[(f-g)^2] \geq 1/8$  for any balanced (hence $1$-resilient) Boolean function $g$. 
%If $\Var[f] < 1/2$, then $\Pr[f = 0] < 0.49$ and $\Pr[f = 1] < 0.49$, so $f$ is $0.01$-far from balanced and thus $0.01$-far from $1$-resilient. 
If $\widehat{f}(\{i\}) > n^{-0.49}$ for some $i$, then $f$ is $\Omega(n^{-0.49})$-far from every Boolean function $g$ where $\widehat{g}(\{i\}) = 0$. 

We assume $\Var[f] \geq 1/2$ and $\widehat{f}(\{i\}) \leq n^{-0.49}$ for all $i \in [n]$.  Since $f$ is monotone, $\Inf_i(f) \leq n^{-0.49}$ for all $i \in [n]$, and by (Talagrand's strengthening of) the KKL Theorem \cite{Talagrand:93,KahnKL:88}, $\Inf(f) \geq K \log n$ for some constant $K$, and $\sum_{i \in [n]} \widehat{f}(\{i\}) \geq K \log n$.  Let $g : \{-1,1\}^n \into \{-1,1\}$ be a $1$-resilient Boolean function; we will show that $\Pr_\bx[f(\bx) \neq g(\bx)] = \Omega(\sqrt{\frac{\log n}{n}})$.

Recall the functions $f_t$ defined earlier: 

\[
f_t(x) =
\begin{cases}
1 & \mathrm{if \;} |x| > t\sqrt{n} \\
0 & \mathrm{if \;} -t\sqrt{n} \leq |x| \leq t\sqrt{n} \\
-1 & \mathrm{if \;} |x| < -t\sqrt{n}
\end{cases}
\]

Select $t$ to be the largest $t$ such that $f_t$ satisfies $\Pr[f_t(\bx) \neq 0] \geq \Pr[(f - g)(\bx) \neq 0] = \Pr[f(\bx) \neq g(\bx)]$.  We then have $K \log n \leq \sum_{i \in [n]} \widehat{f-g}(\{i\}) \leq \sum_{i \in [n]} \widehat{f_t}(\{i\})$, where the second inequality holds because $f_t$ maximizes the sum of the linear coefficients for any function with support size $\Pr[f_t(\bx)\neq 0]$, and the support size of $f_t$ is at least the support size of $f-g$.

\ignore{
We use three properties appearing in \cite{OWimmer:09} that follow from error estimates for the Central Limit Theorem \cite{Feller}: for large enough $n$ and $\sqrt{\log n}/100 < t < n^{1/10}$, we have

\begin{align}
\phi(t)\sqrt{n}/3 & \leq \Inf(f_t) \leq 3\phi(t)\sqrt{n} \label{eqn:prop11} \\
\h\Pr_{\bx}[f_t(\bx) \neq 0] & =\E_{\bx}[f_t(\bx)\mathrm{Maj}(\bx)] \geq \phi(t)/(3t) \label{eqn:prop12} \\
\Pr[|\bx| = t] & \leq 4\phi(t)/\sqrt{n} \label{eqn:prop13}.
\end{align}

where $\phi$ is the probability density function of the standard Gaussian distribution: $\phi(u) = \frac{1}{2\pi} \exp(-u^2/2)$.}

Again, because $f_t$ is monotone, $\Inf(f_t) = \sum_{i \in [n]} \widehat{f_t}(\{i\})$.  Equation~\ref{eqn:prop1} implies that $(3K \log n)/\sqrt{n} \geq \phi(t) \geq (K \log n)/(3\sqrt{n})$, and it follows that $t \leq 4\sqrt{\log n}$.  From Equation~\ref{eqn:prop2}, we have $\Pr_{\bx}[f_t(\bx) \neq 0] \geq (4K/3)\sqrt{\frac{\log n}{n}}$.  By the choice of $t$, we have 
%Also, by the bounds of $t$, we have
\begin{align*}
\Pr_{\bx}[f(\bx) \neq g(\bx)] & >  \Pr_\bx[f_{t+1}(\bx) \neq 0] \\
& \geq \Pr_\bx[f_t(\bx)\neq 0] - 2\Pr_\bx[|\bx| = t] \\ 
& \geq \frac{4K}{3}\sqrt{\frac{\log n}{n}} - 24K\frac{\log n}{n} = \Omega\left(\sqrt{\frac{\log n}{n}}\right),
\end{align*}

where the first inequality is an application of the union bound, and the second is an application of Equation~\ref{eqn:prop3}. 
\end{proof}

\ignore{
(Notes)\\
Let $\ell_x$ be the maximal run size for any $x\in \bits$. Let $R(x)$ be the set of bits occurring in some maximal run.  We have that $\CycleRun(x) \neq \CycleRun(x^{\oplus i})$ if and only if is contained in exactly one of $R(x)$ and $R(x^{\oplus i})$.  Thus we may obtain the desired influence bound by showing that $\E[ R(x)] \leq O(\log n)$.    

For a given $x$ and $\ell_x$, let $r_\ell$ be the number runs of length $\ell$. We can easily bound the contribution to the expectation when $\ell_x>\log n$ (Markov)  or when $\ell_x < \log n-100\log \log n$ (coupling/geometric distribution).  Let $u=100\log \log n$.

Then we want to bound 
\begin{align*}
\sum_{\log n-u \leq \ell \leq \log n} \sum_{r} \ell \cdot r \cdot \Pr_x[ \ell_x = \ell,r_\ell > r]  \leq & \\
\sum_{\log n-u \leq \ell \leq \log n} \ell \cdot \sum_{r} r \cdot \Pr_x[ r_\ell > r]  \leq & \\
\end{align*}
Again using the geometric sampling, we can use a bound of 
$$ \binom{n}{r} 2^{-\ell r} (1-2^{1-\ell})^{n-r}$$
for that probability (as long as we're summing over all $r$).
}